\documentclass{elsarticle}

\journal{Journal of \LaTeX\ Templates}

\usepackage{hyperref}
\usepackage{helvet}
\usepackage{courier}
\usepackage{amssymb}
\usepackage{amsmath}
\usepackage{algorithm}
\usepackage[noend]{algpseudocode}
\usepackage{graphicx}
\usepackage{subcaption}
\usepackage[english]{babel}
\usepackage{mathrsfs}
\usepackage{amscd}
\usepackage{stmaryrd}
\usepackage{todonotes}
\usepackage{wrapfig}
\usepackage{booktabs}
\usepackage{multirow}
\usepackage{color}

\newtheorem{definition}{Definition}
\newtheorem{corollary}{Corollary}
\newtheorem{theorem}{Theorem}
\newtheorem{lemma}{Lemma}
\newtheorem{example}{Example}
\newenvironment{proof}{\noindent \textbf{Proof:}}{\hfill  $\boxempty$}

\newcommand{\powerset}[1]{\mathcal{P}(#1)}

\newcommand{\be}{\begin{itemize}}
\newcommand{\ee}{\end{itemize}}

\newcommand{\inputdomain}{{\rm D}}
\newcommand{\manipulation}{\delta}
\newcommand{\manipulationset}{\Delta}
\newcommand{\inputImage}{\alpha}
\newcommand{\distance}[2]{||#2||_{#1}}

\newcommand{\lipschitzConstant}{\hbar}

\newcommand{\natureNumber}{\mathbb{N}}

\newcommand{\commentout}[1]{}
\newcommand{\response}[1]{{#1}}

\newcommand{\instruction}{\psi}
\newcommand{\instructionset}{\Psi}
\newcommand{\maximumsaferadius}{{\tt MSR}}
\newcommand{\featurerobustness}{{\tt FR}}
\newcommand{\xfeaturerobustness}{{\tt xFR}}

\newcommand{\finitemaximumsaferadius}{{\tt FMSR}}
\newcommand{\finitefeaturerobustness}{{\tt FFR}}
\newcommand{\xfinitefeaturerobustness}{{\tt xFFR}}

\newcommand{\lowerbound}{\textsf{l}}
\newcommand{\upperbound}{\textsf{u}}
\newcommand{\depsilon}{d^{\epsilon}}

\newcommand{\tauimage}{G}

\begin{document}

\begin{frontmatter}

\title{A Game-Based Approximate Verification of Deep Neural Networks with Provable Guarantees}

\author[oxford]{Min Wu}
\ead{min.wu@cs.ox.ac.uk}

\author[georgia]{Matthew Wicker\fnref{fn1}}
\ead{matthew.wicker25@uga.edu}
\fntext[fn1]{Matthew conducted this research when he was a visiting undergraduate student at the University of Oxford.}

\author[oxford]{Wenjie Ruan}
\ead{wenjie.ruan@cs.ox.ac.uk}

\author[liverpool]{Xiaowei Huang}
\ead{xiaowei.huang@liverpool.ac.uk}

\author[oxford]{Marta Kwiatkowska\corref{cor1}}
\cortext[cor1]{Corresponding author}
\ead{marta.kwiatkowska@cs.ox.ac.uk}

\address[oxford]{University of Oxford, UK}
\address[georgia]{University of Georgia, USA}
\address[liverpool]{University of Liverpool, UK}

\begin{abstract}
Despite the improved accuracy of deep neural networks, the discovery of adversarial examples has raised serious safety concerns.
In this paper, we study two variants of pointwise robustness, the 
\emph{maximum safe radius} problem, which for a given input sample computes the minimum distance to an adversarial example, and the \emph{feature robustness} problem, which aims to quantify the robustness of individual features to adversarial perturbations. 
We demonstrate that, under the assumption of Lipschitz continuity, both problems can be approximated using finite optimisation by discretising the input space, and the approximation  has provable guarantees, i.e., the  error  is bounded. We then show that the resulting 
optimisation problems can be reduced to the solution of two-player turn-based 
games, where the first player selects features and the second perturbs the image within the feature. While the second player aims to minimise the distance to an adversarial example, depending on the optimisation objective the first player can be cooperative or competitive.
We employ an anytime approach to solve the games, 
in the sense of approximating the value of a game by monotonically improving its upper and lower bounds. 
The Monte Carlo tree search algorithm is applied to compute upper bounds for both games, and the Admissible A* and the Alpha-Beta Pruning algorithms are, respectively, used to compute lower bounds for the maximum safety radius and feature robustness games. 
When working on the upper bound of the maximum safe radius problem, our tool demonstrates competitive performance against existing adversarial example crafting algorithms. 
Furthermore, we show how our framework can be deployed to evaluate pointwise robustness of neural networks in safety-critical applications such as traffic sign recognition in self-driving cars. 

\end{abstract}


\end{frontmatter}

\section{Introduction}

\newcommand{\DeepGame}{\mathsf{DeepGame}}
\newcommand{\terminate}{t}
\newcommand{\reward}{R}
\newcommand{\playerOne}{{\tt I}}
\newcommand{\playerTwo}{{\tt II}}
\newcommand{\opt}{{\tt opt}}

Deep neural networks (DNNs or networks, for simplicity) have been developed  for 
a variety of tasks, including malware detection~\cite{malware}, abnormal network activity detection~\cite{ryan:nips10}, and self-driving cars~\cite{NVIDIA,road-segmentation,traffic-classification-lecun}. A classification network $N$ can be  used as a decision-making algorithm: given an input $\inputImage$, it suggests a decision $N(\inputImage)$ among a set of possible decisions. While the accuracy of neural networks has greatly improved, matching the cognitive ability of humans~\cite{LBH2015}, they are susceptible to adversarial examples~\cite{Biggio2013,propertiesOfNeuralNetworks}. 
An adversarial example is an input which, though initially classified correctly, is misclassified after a minor, perhaps imperceptible, perturbation. 
Adversarial examples pose challenges for 
self-driving cars, where neural network solutions have been proposed for tasks such as end-to-end steering~\cite{NVIDIA}, road segmentation~\cite{road-segmentation}, and traffic sign classification~\cite{traffic-classification-lecun}. 
In the context of steering and road segmentation, an adversarial example may cause a car to steer off the road or drive into barriers, and misclassifying traffic signs may cause a vehicle to drive into oncoming traffic. 
Figure~\ref{fig:coverimage} shows an image of a traffic light correctly classified by a state-of-the-art network, which is then misclassified after only a few pixels have been changed.
Though somewhat artificial, since in practice the controller would rely on additional sensor input when making a decision, such cases strongly suggest that, before deployment in safety-critical tasks, DNNs' resilience (or robustness) to adversarial examples must be strengthened.

\begin{figure}
    \centering
    \includegraphics[width=0.6\textwidth]{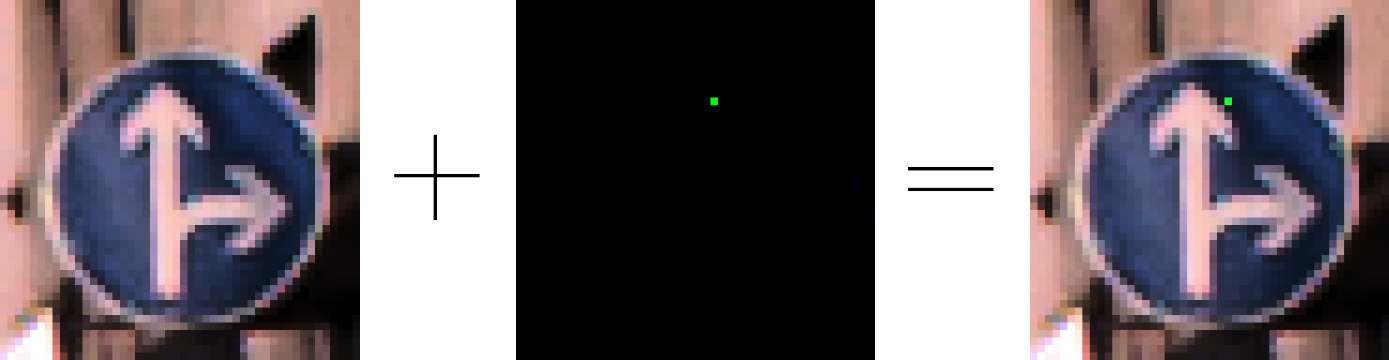}
    \caption{An adversarial example for a neural network trained on the GTSRB dataset. \response{After a slight perturbation of Euclidean distance 0.88}, the image classification changes from ``go \textbf{right} or straight'' to ``go \textbf{left} or straight''.}
    \label{fig:coverimage}
\end{figure}

Robustness of neural networks is an active topic of investigation and 
a number of approaches have been proposed to search for adversarial examples (see Related Work). They are based on computing the
gradients~\cite{FGSM}, along which a heuristic search moves; computing a Jacobian-based saliency map~\cite{JSMA}, based on which pixels are selected to be changed; transforming the existence of adversarial examples into an optimisation problem~\cite{CW-Attacks}, on which an optimisation algorithm can be applied; transforming the existence of adversarial examples into a constraint solving problem~\cite{KBDJK2017}, on which a constraint solver can be applied; or discretising the neighbourhood of a point and searching it exhaustively in a layer-by-layer manner~\cite{DLV}. 

In this paper, we propose a novel game-based approach for safety verification of DNNs. We consider two pointwise robustness  
problems, referred to as the  \emph{maximum safe radius} problem and  \emph{feature robustness} problem, respectively. The former 
aims to compute for a given input the minimum distance to an adversarial example, and therefore can be regarded as the computation of an \emph{absolute} safety radius, within which no adversarial example exists. The latter 
problem 
studies whether the crafting of adversarial examples can be controlled by restricting perturbations to only certain features (disjoint sets of input dimensions), and therefore can be seen as the computation of a \emph{relative} safety radius, within which the existence of adversarial examples is controllable. 

Both pointwise robustness problems are formally expressed in terms of non-linear optimisation, 
\response{which is computationally challenging for realistically-sized networks. We thus utilise Lipschitz continuity of DNN layers, which bounds the maximal rate of change of outputs of a function with respect to the change of inputs, as proposed for neural networks with differentiable layers in \cite{old-lip-1,old-lip-2}.
This enables safety verification by relying on Lipschitz constants to provide guaranteed bounds on DNN output \emph{for all} possible inputs. 
We work with modern DNNs whose layers, e.g., ReLU, may not be differentiable, and 
reduce the verification to 
finite optimisation. 
More precisely,} we prove that \response{under the assumption of Lipschitz continuity~\cite{RHK2018}} it is sufficient to consider a finite number of uniformly sampled inputs 
when the distances between the inputs are small, and that this reduction has provable guarantees, in the sense of the error being bounded by the distance between sampled inputs.
%

We then show that the finite optimisation problems can be computed as the solution of 
two-player turn-based 
games, where Player~$\playerOne$ selects features and Player~$\playerTwo$ then performs 
a perturbation within the selected features.  
After both players have made their choices, 
the input is perturbed and the game continues.
While Player~$\playerTwo$ aims to minimise the distance to an adversarial example, Player~$\playerOne$ can be \emph{cooperative} or \emph{competitive}.
When it is cooperative, the optimal reward of  Player~$\playerOne$ is equal to the maximum safe radius. On the other hand, when it is competitive the optimal reward of Player~$\playerOne$ quantifies feature robustness. Finally, because the state space of the game models is intractable, we employ an anytime approach to compute the upper and lower bounds of Player~$\playerOne$ optimal reward. The anytime approach ensures that the bounds can be gradually, but strictly, improved so that they eventually converge. More specifically, we apply Monte Carlo tree search algorithm to compute the upper bounds for both games, and Admissible A* and Alpha-Beta Pruning, respectively, to compute the lower bounds for the games. 


We implement the method in a software tool $\DeepGame$\footnote{The software package is available from  \url{https://github.com/TrustAI/DeepGame}}, 
and 
%
%
conduct experiments on DNNs to show convergence of lower and upper bounds for the {maximum safe radius} and {feature robustness} problems. 
Our approach can be configured to work with a variety of feature extraction methods that partition the input, for example image segmentation, with simple adaptations. 
For the image classification networks we consider in the experiments, we employ both the saliency-guided $\mathsf{grey}$-box approach adapted from \cite{RWSHKK2018} and the feature-guided $\mathsf{black}$-box method based on the SIFT object detection technique~\cite{SIFT}. 
For the maximum safety radius problem, our experiments show that, on networks trained on the benchmark datasets such as MNIST~\cite{MNIST}, CIFAR10~\cite{CIFAR10} and GTSRB~\cite{GTSRB},
the upper bound computation method is competitive 
with state-of-the-art heuristic methods (i.e., without provable guarantees) that rely on white-box saliency matrices or sophisticated optimisation procedures.
%
Finally, to show that our framework is well suited to safety testing and decision support for deploying DNNs in safety-critical applications, we 
experiment
on state-of-the-art networks, including the winner of the Nexar traffic light challenge~\cite{NexarData}. 

The paper significantly extends work published in \cite{WHK2017}, where the game-based approach was first introduced for the case of cooperative games and evaluated on the computation of upper bounds for the maximum safety radius problem using the SIFT feature extraction method. 
In contrast, in this paper we additionally study feature robustness, generalise the game to allow for the competitive player, and develop algorithms for the computation of both lower and upper bounds. We also give detailed proofs of the theoretical guarantees and error bounds.

The structure of the paper is as follows. After introducing preliminaries in Section~\ref{sec:preliminaries}, we formalise the {maximum safety radius} and {feature robustness} problems 
in Section~\ref{sec:SafetyThms}. We present our game-based approximate verification approach and state 
the guarantees in Section~\ref{sec:approach}.
Algorithms and implementation are described in Section~\ref{sec:algorithms}, while experimental results are given in Section~\ref{sec:results}. We discuss the related work in Section~\ref{sec:related} and conclude the paper in Section~\ref{sec:concl}.

\section{Preliminaries}\label{sec:preliminaries}

\newcommand{\severity}{sev}
\newcommand{\expectation}{{\tt E}}
\newcommand{\GaussianMixture}{{\cal G}}
\newcommand{\feature}{\lambda}
\newcommand{\setoffeatures}{\Lambda}
\newcommand{\strategy}{\sigma}

Let $N$ be a neural network with a set $C$ of classes. Given an input $\inputImage$ and a class $c \in C$, we use $N(\inputImage,c)$ to denote the confidence (expressed as a probability value obtained from normalising the score) of $N$ believing  that $\inputImage$ is in class $c$. Moreover, we write $N(\inputImage) = \arg\max_{c\in C} N(\inputImage,c)$ for the class into which $N$ classifies $\inputImage$. 
We let $P_0$ be the set of input dimensions, $n = |P_0|$ be the number of input dimensions, and remark that without loss of generality the dimensions of an input are normalised as real values in $[0,1]$. The input domain is thus a vector space 
$$
\inputdomain = [0,1]^{n}.
$$
For image classification networks, the input domain $\inputdomain$ 
can be represented as
$[0,1]_{[0,255]}^{w\times h\times ch}$, where $w,h,ch$ are the width, height, and number of channels of an image, respectively. That is, we have $P_0 = w\times h\times ch$.
We may refer to an element in $w \times h$ as a \textit{pixel} and an element in $P_0$ as a \textit{dimension}. 
\response{We use $\inputImage[i]$ for $i\in P_0$ to denote the value of the $i$-th dimension of $\inputImage$.}

\subsection{Distance Metric and Lipschitz Continuity}

As is common in the field, we will work \response{with} $L_k$ distance functions to measure the distance between inputs,
denoted $\distance{k}{\inputImage-\inputImage'}$ with $k\geq 1$, and satisfying the standard axioms of a metric space:
\begin{itemize}
    \setlength\itemsep{0em}
    \item $\distance{k}{\inputImage-\inputImage'}\geq 0$ (non-negativity),
    \item $\distance{k}{\inputImage-\inputImage'}=0$ implies that $\inputImage = \inputImage'$ (identity of indiscernibles),
    \item $\distance{k}{\inputImage-\inputImage'}= \distance{k}{\inputImage'-\inputImage}$ (symmetry),
    \item $\distance{k}{\inputImage-\inputImage''} \leq \distance{k}{\inputImage-\inputImage'} + \distance{k}{\inputImage'-\inputImage''}$ (triangle inequality).
\end{itemize}
While we focus on $L_k$ distances, including
$L_1$ (Manhattan distance), $L_2$ (Euclidean distance), and $L_\infty$ (Chebyshev distance), we emphasise that 
the results of this paper hold for any distance metric 
and can be adapted to image similarity distances such as SSIM~\cite{WSB2003}.
\response{Though our results do not generalise to $L_0$ (Hamming distance), we utilise it for the comparison with existing approaches to generate adversarial examples, i.e., without provable guarantees (Section~\ref{subsec:L0Comparison}).}

Since we work with pointwise robustness~\cite{SZSBEGF2014}, we need to consider the \emph{neighbourhood} of a given input.
\begin{definition}\label{def:inputregion}
Given an input $\inputImage$, a distance function $L_k$, and a distance $d$, we define 
the \emph{d-neighbourhood} $\eta(\inputImage,L_k,d)$ of $\inputImage$ wrt $L_k$
$$\eta(\inputImage,L_k,d)=\{\inputImage' ~|~ \distance{k}{\inputImage'-\inputImage} \leq d\}$$
as the set of inputs whose distance to $\inputImage$ is no greater than $d$ with respect to $L_k$.  
\end{definition}

The $d$-neighbourhood of $\inputImage$ is simply the $L_k$ ball with radius $d$. For example, 
$\eta(\inputImage,L_1,d)$ includes those inputs such that the sum of the differences of individual dimensions from the original input $\inputImage$ is no greater than $d$, i.e., 
\response{$\distance{1}{\inputImage'-\inputImage} = \sum_{i\in P_0}| \inputImage[i] - \inputImage'[i]|$}. Furthermore, we have
\response{$\distance{2}{\inputImage'-\inputImage} = \sqrt{\sum_{i\in P_0} (\inputImage[i] - \inputImage'[i])^2}$}
and
\response{$\distance{\infty}{\inputImage'-\inputImage} = \max_{i\in P_0} |\inputImage[i] - \inputImage'[i]|$}.
We will sometimes work with $\depsilon$-neighbourhood, where, given a number $d$, $\depsilon = d + \epsilon$ for any real number $\epsilon>0$ denotes a number greater than $d$.

We will restrict the neural networks we consider to those that satisfy the \emph{Lipschitz continuity} assumption, noting that 
all networks whose inputs are bounded, including all image classification networks we studied, are Lipschitz continuous. {Specifically, it is shown in~\cite{SZSBEGF2014,RHK2018} that most known types of layers, including fully-connected, convolutional, ReLU, maxpooling, sigmoid, softmax, etc., are Lipschitz continuous.}  

\begin{definition}\label{def:Lipschitz}
Network $N$ is a Lipschitz network with respect to distance function $L_k$ if there exists a constant $\lipschitzConstant_c > 0 $ for every class $c\in C$ such that, for all $\inputImage,\inputImage'\in \inputdomain$, we have 
\begin{equation}
|N(\inputImage', c)-N(\inputImage,c)| \leq \lipschitzConstant_c \cdot \distance{k}{\inputImage' - \inputImage},
\end{equation}
\response{where $\lipschitzConstant_c$ is the \emph{Lipschitz constant} for class $c$.}
\end{definition}


\subsection{Input Manipulations} 
To study the crafting of adversarial examples, we require the following operations for manipulating inputs. 
Let 
$\tau>0$ be a positive real number representing the manipulation magnitude, then we can define \emph{input manipulation} operations 
\response{
$\manipulation_{\tau,X,\instruction}: \inputdomain \rightarrow \inputdomain$ for $X\subseteq P_0$, a subset of input dimensions,
and $\instruction:P_0\rightarrow \natureNumber$,  an instruction function by:  %
\begin{equation}
    \manipulation_{\tau,X,\instruction}(\inputImage[i]) =
    \begin{cases}
        \inputImage[i] + \instruction(i) * \tau,    & \quad \text{if } i\in X \\
        \inputImage[i],                             & \quad \text{otherwise}
    \end{cases}
\end{equation}
for all $i\in P_0$.
Note that if the values are bounded, e.g., in the interval $[0,1]$, then  $\manipulation_{\tau,X,\instruction}(\inputImage[i])$ needs to be restricted to be within the bounds.  
Let $\instructionset$ be the set of possible instruction functions.}

\commentout{
$$
\manipulation_{\tau,X,\instruction}(\inputImage)(j) = \left\{
\begin{array}{ll}
\inputImage(j) + \instruction(j) * \tau, & \text{if } j\in X 
\\
\inputImage(j),  & \text{otherwise}\\
\end{array}
\right.
$$
for all $j\in P_0$.
Note that, if the values are bounded, e.g., $[0,1]$, then  $\manipulation_{\tau,X,\instruction}(\inputImage)(j)$ needs to be restricted to be within the bounds.  
Let $\instructionset$ be the set of possible instruction functions.
}


The following lemma shows that input manipulation operations allow one to map one input to another 
by changing the values of input dimensions, regardless of the distance measure $L_k$. 

\begin{lemma}
Given any two inputs $\inputImage_1$ and $\inputImage_2$, and a  distance $\distance{k}{\inputImage_1 - \inputImage_2}$ for any measure $L_k$, there exists a magnitude $\tau>0$, an instruction function $\instruction \in \instructionset$, and a subset $X \subseteq P_0$ of input dimensions, such that 
    $$\distance{k}{\inputImage_2 - \manipulation_{\tau,X,\instruction}(\inputImage_1)} \leq \epsilon$$ 
where $\epsilon>0$ is an error bound. 
\end{lemma}

Intuitively, any distance can be implemented through an input manipulation with an error bound $\epsilon$.
\response{
The error bound $\epsilon$ is needed because input $\alpha \in \inputdomain = [0,1]^{n}$ is bounded, and thus reaching another precise input point via a manipulation is difficult when each input dimension is a real number.}

%
We will also distinguish a subset of \emph{atomic} input manipulations, each of which changes a single dimension for a single magnitude. 
\begin{definition}\label{def:atomicManipulation}
Given a set $X$, we let $\manipulationset(X)$ be the set of \emph{atomic} input manipulations $\manipulation_{\tau,X_1,\instruction_1}$ such that 
\begin{itemize}
    \setlength\itemsep{0em}
    \item $X_1\subseteq X$ and $|X_1| = 1$, and 
    \item $\instruction_1(i)\in \{-1,+1\}$ for all $i\in P_0$. 
\end{itemize}
\end{definition}

\begin{lemma}\label{lemma:atomicinput}
Any input manipulation $\manipulation_{\tau,X,\instruction}(\inputImage)$ for some $X$ and $\instruction$ can be implemented with a finite sequence of input manipulations  $\manipulation_{\tau,X_1,\instruction_1}(\inputImage), ..., \manipulation_{\tau,X_m,\instruction_m}(\inputImage) \in \manipulationset(X)$.
\end{lemma}

While the existence of a sequence of atomic manipulations implementing a given manipulation is determined, there may exist multiple sequences. On the other hand, from a given sequence of atomic manipulations we can construct a single input manipulation 
by sequentially applying the atomic manipulations.

\subsection{Feature-Based Partitioning}

Natural data, for example natural images and sounds, forms a high-dimensional manifold, which embeds tangled manifolds to represent their features~\cite{Carlsson2008}. 
Feature manifolds usually have lower dimensions than the data manifold.
%
Intuitively, the set of features form a partition of the input dimensions $P_0$.
In this paper, we use a variety of feature extraction methods to partition the set $P_0$ into disjoint subsets.  

\begin{definition}\label{def:features}
Let $\feature$ be a feature of an input $\inputImage\in\inputdomain$, then we use $P_{\feature}\subseteq P_0$ to denote the dimensions represented by $\feature$. Given an input $\inputImage$, 
a \emph{feature extraction} function  $\setoffeatures$ maps an input $\inputImage$ into a set of features $\setoffeatures(\inputImage)$ such that (1) $P_0 = \bigcup_{\feature\in \setoffeatures(\inputImage)} P_\feature$, and (2) $P_{\feature_i} \cap P_{\feature_j} = \emptyset $ for any $\feature_i,\feature_j \in \setoffeatures(\inputImage)$ with $i \neq j$. 
\end{definition}

We remark that our technique is not limited to image classification networks and is able to work with general classification tasks, as long as there is a suitable feature extraction method that generates a partition of the input dimensions. 
In our experiments we focus on image classification for illustrative purposes and to enable better comparison, and employ \emph{saliency-guided} $\mathsf{grey}$-box and \emph{feature-guided} $\mathsf{black}$-box approaches to extract features, described in Section~\ref{sec:approach}.

\section{Problem Statement}
\label{sec:SafetyThms}

In this paper we focus on \emph{pointwise robustness}~\cite{SZSBEGF2014}, defined as the invariance of the network's classification over a small neighbourhood of a given input. This is a key concept, which also allows one to define robustness as a network property, by averaging with respect to the distribution of the test data set. Pointwise robustness can be used to define \emph{safety} of a classification decision for a specific input, understood as the non-existence of an adversarial example in a small neighbourhood of the input. We work with this notion and consider two problems for quantifying the robustness of the decision, the computation of the \emph{maximum safe radius} and \emph{feature robustness}, which we introduce next.


First we recall the concept of an \emph{adversarial example}, as well as what we mean by \emph{targeted} and \emph{non-targeted safety}. 

\begin{definition}\label{def:constraints}
Given an input $\inputImage\in\inputdomain$, a distance measure $L_k$ for some $k\geq 0$, and a distance $d$, an \emph{adversarial example} $\inputImage'$ of class $c$ is such that 
$\inputImage'\in \eta(\inputImage,L_k,d)$, 
$N(\inputImage)\neq N(\inputImage')$, and 
$N(\inputImage')=c$. 
Moreover, we write $adv_{k,d}(\inputImage,c)$ for the set of adversarial examples of class $c$ and let $$adv_{k,d}(\inputImage)=\bigcup_{c\in C, c\neq N(\inputImage)}adv_{k,d}(\inputImage,c).$$ A \emph{targeted safety} of class $c$ is defined as $adv_{k,d}(\inputImage,c)=\emptyset$, and \emph{non-targeted safety} is $adv_{k,d}(\inputImage)=\emptyset$.  
\end{definition}

The following formalisation focuses on targeted safety of a fixed input $\inputImage$ and a fixed class $c \neq N(\inputImage)$ for a network $N$. The case of non-targeted safety (misclassification into class other than $c$) is similar.

\begin{figure}[t]
	\centering
	\includegraphics[width=0.75\linewidth]{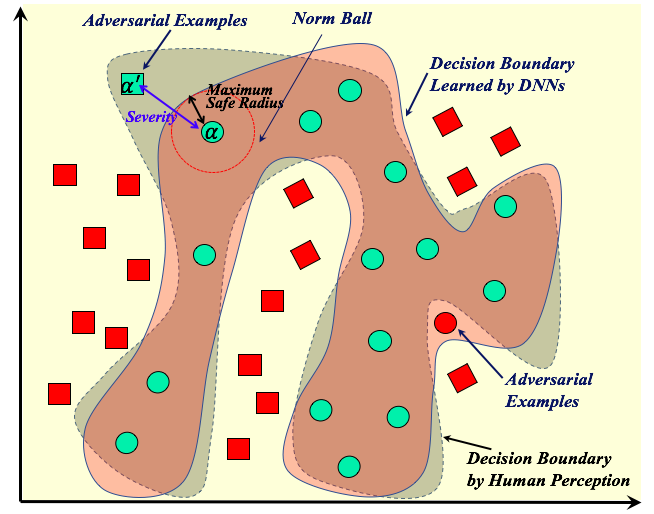}
	\caption{The {\em Maximum Safe Radius} ($\maximumsaferadius$) problem aims to quantify the minimum distance from an original image $\inputImage$ to an adversarial example, equivalent to finding the radius of a maximum safe norm ball. The solid line represents the classification boundary learned by a DNN, while the dashed line is the decision boundary. 
	Adversarial examples tend to lie where the decision and classification boundaries do not align.
	Intuitively, finding an adversarial example (green square) 
	can only provide a loose upper bound of $\maximumsaferadius$. 
}\label{fig:LeastSeverity}
\end{figure}

\subsection{The Maximum Safe Radius Problem}

%
\commentout{
We regard the distance to an adversarial example as a measure of its severity, i.e., let 
\begin{equation}
 \distance{k}{\inputImage - \inputImage'}
\end{equation} be the \emph{severity} of the adversarial example $\inputImage'$ against the original image $\inputImage$.}
Given a targeted safety problem for $\inputImage$, we aim to compute the distance $\distance{k}{\inputImage - \inputImage'}$ to the nearest adversarial example within the $d$-neighbourhood of $\inputImage$, or in other words the radius of the maximum safe ball, illustrated in Figure~\ref{fig:LeastSeverity}.



\begin{definition}[Maximum Safe Radius] \label{def:objective}
The \emph{maximum safe radius} problem is to compute the minimum distance from the original input $\inputImage$ to an adversarial example, i.e., 
\begin{equation}\label{equ:lsobjective}
\maximumsaferadius(k,d,\inputImage,c) = \min_{\inputImage'\in \inputdomain} \{ \distance{k}{\inputImage - \inputImage'} \mid \inputImage' \in adv_{k,d}(\inputImage,c)\}
\end{equation}
If $adv_{k,d}(\inputImage,c)=\emptyset$, we let $\maximumsaferadius(k,d,\inputImage,c)  = \depsilon $. 
\end{definition}
Intuitively, $\maximumsaferadius(k,d,\inputImage,c)$ represents an \emph{absolute} safety radius within which all inputs are safe. In other words, within a 
distance of less than $\maximumsaferadius(k,d,\inputImage,c)$, no adversarial example is possible. 
When no adversarial example can be found within radius $d$, i.e., $adv_{k,d}(\inputImage,c)=\emptyset$, the maximum safe radius cannot be computed, but is  definitely greater than $d$. Therefore, we let $\maximumsaferadius(k,d,\inputImage,c) = \depsilon $. 

Intuitively, finding an adversarial example 
	can only provide a loose upper bound of $\maximumsaferadius$. Instead, this paper investigates a more 
	fundamental problem -- how to approximate the {\em true} $\maximumsaferadius$ distance with provable guarantees.

\paragraph{Approximation Based on Finite Optimisation}

Note that the sets $adv_{k,d}(\inputImage,c)$ and $adv_{k,d}(\inputImage)$ of adversarial examples can be infinite. 
We now present a discretisation method that allows us to approximate the maximum safe radius using finite optimisation, and show that such a reduction has provable guarantees, provided that the network is Lipschitz continuous. 
\response {Our approach proceeds by constructing a finite `grid' of points in the input space. Lipschitz continuity enables us to reduce the verification problem to manipulating just the grid points, through which we can bound the output behaviour of a DNN on the whole input space, since Lipschitz continuity ensures that the network behaves well within each cell. The number of grid points is inversely proportional to 
the Lipschitz constant. 
However, estimating a tight Lipschitz constant is difficult, and so, rather than working with the Lipschitz constant directly, we assume the existence of a (not necessarily tight) Lipschitz constant and work instead with a chosen fixed magnitude of an input manipulation, $\tau \in (0,1]$. 
We show how to determine the largest $\tau$ for a given Lipschitz network and give error bounds for the computation of $\maximumsaferadius$ that depend on $\tau$. 
We discuss how Lipschitz constants can be estimated in Section \ref{lip-discussion}  and Related Work.} 

\response{We begin by constructing, for a chosen fixed magnitude $\tau \in (0,1]$, input manipulations to search for adversarial examples.} 


\begin{definition}\label{def:FMSR}
Let $\tau \in (0,1]$ be a manipulation magnitude. The \emph{finite maximum safe radius} problem $\finitemaximumsaferadius(\tau,k,d,\inputImage,c)$ based on input manipulation is as follows: 
\begin{equation}\label{equ:lsobjective2}
\response{
    \min_{\setoffeatures'\subseteq \setoffeatures(\inputImage)}
    \min_{X\subseteq \bigcup_{\feature\in \setoffeatures'}P_\feature}
    \min_{\instruction \in \instructionset}
    \{ \distance{k}{\inputImage - \manipulation_{\tau,X,\instruction}(\inputImage)} \mid \manipulation_{\tau,X,\instruction}(\inputImage) \in adv_{k,d}(\inputImage,c) \}.
}
\end{equation}
If $adv_{k,d}(\inputImage,c)=\emptyset$, we let $\finitemaximumsaferadius(\tau,k,d,\inputImage,c) = \depsilon$. 
\end{definition}

Intuitively, we aim to find a set $\setoffeatures'$ of features, a set $X$ of dimensions within $\setoffeatures'$, and a manipulation instruction $\instruction$ such that the application of the atomic manipulation $\manipulation_{\tau,X,\instruction}$ on the original input $\inputImage$ leads to an adversarial example $\manipulation_{\tau,X,\instruction}(\inputImage)$ that is nearest to $\inputImage$ among all adversarial examples. Compared to Definition~\ref{def:objective}, the search for another input by $\min_{\inputImage'\in \inputdomain}$ over an infinite set is implemented by minimisation over the finite sets of feature sets and instructions.  

Since the set of input manipulations 
is finite for a fixed magnitude, the above optimisation problems need only explore a finite number of `grid' points in the input domain $\inputdomain$. We have the following lemma. 

\begin{lemma}\label{lemma:onedirection}
For any $\tau\in(0,1]$, we have that 
$\maximumsaferadius(k,d,\inputImage,c) \leq \finitemaximumsaferadius(\tau,k,d,\inputImage,c)$. 
\end{lemma}

\response{To ensure the lower bound of $\maximumsaferadius(k,d,\inputImage,c)$ in Lemma~\ref{lemma:onedirection},}
we utilise the fact that the network is Lipschitz continuous~\cite{RHK2018}. 
First, we need the concepts of a $\tau$-grid input, for a manipulation magnitude $\tau$, and a misclassification aggregator. 
\response{The intuition for the $\tau$-grid is illustrated in Figure~\ref{fig:guarantee}. We construct a finite set of grid points uniformly spaced by $\tau$ in such a way that they can be covered by small subspaces centred on grid points. We select a sufficiently small value for $\tau$ based on a given Lipschiz constant so that all points in these subspaces are are classified the same. We then show that an optimum point on the grid is within an error bound dependent on $\tau$ from the true optimum, i.e., the closest adversarial example.}


\begin{definition}
An image $\inputImage' \in \eta(\inputImage,L_k,d)$ is a \emph{$\tau$-grid input} if for all dimensions $p\in P_0$ we have $|\inputImage'(p)-\inputImage(p)| = n * \tau$ for some $n\geq 0$. Let $\tauimage(\inputImage,k,d)$ be the set of $\tau$-grid inputs in $\eta(\inputImage,L_k,d)$. 
\end{definition}

We note that $\tau$-grid inputs in the set $\tauimage(\inputImage,k,d)$ are reachable from each other by applying an input manipulation. 
The main purpose of defining $\tau$-grid inputs is to ensure that the space $\eta(\inputImage,L_k,d)$ can be covered by small subspaces \response{centred on grid points}. To implement this, we need the following lemma. 
\begin{lemma}\label{lemma:cover}
We have $\eta(\inputImage,L_k,d) \subseteq \bigcup_{\inputImage' \in \tauimage(\inputImage,k,d)}\eta(\inputImage',L_k,\frac{1}{2}d(k,\tau)) $, where $d(k,\tau)=(|P_0|\tau^k)^{\frac{1}{k}}$.
\end{lemma}
\begin{proof}
Let $\inputImage_1$ be any point in $\eta(\inputImage,L_k,d)$. We need to show  $\inputImage_1\in \eta(\inputImage',L_k,\frac{1}{2}d(k,\tau))$ for some $\tau$-grid input $\inputImage'$. Because every point in $\eta(\inputImage,L_k,d)$ belongs to a $\tau$-grid cell, we assume that $\inputImage_1$ is in a $\tau$-grid cell which, without loss of generality, has a set $T$ of $\tau$-grid inputs as its vertices. Now for any two $\tau$-grid inputs $\inputImage_2$ and $\inputImage_3$ in $T$, we have that $\distance{k}{\inputImage_2-\inputImage_3}\leq d(k,\tau)$, by the construction of the grid. Therefore, we have $\inputImage_1 \in \eta(\inputImage',L_k,\frac{1}{2}d(k,\tau))$ for some $\inputImage' \in T$.
\end{proof}

\response{As shown in Figure~\ref{fig:guarantee}, the distance $\frac{1}{2}d(k,\tau)$ is the radius of norm ball subspaces covering the input space.}  
It is easy to see that $d(1,\tau) = |P_0|\tau$, $d(2,\tau) = \sqrt{|P_0|\tau^2}$, and $d(\infty,\tau) = \tau$.

%
%
\begin{definition}
An input $\inputImage_1\in \eta(\inputImage,L_k,d)$ is a \emph{misclassification aggregator} with respect to a number $\beta>0$ if,  for any $\inputImage_2\in \eta(\inputImage_1,L_k,\beta)$, we have that $N(\inputImage_2) \neq N(\inputImage)$ implies $N(\inputImage_1) \neq  N(\inputImage)$.
\end{definition} 

Intuitively, if a misclassification aggregator $\inputImage_1$ with respect to $\beta$ is classified correctly, then \emph{all} inputs in $\eta(\inputImage_1,L_k,\beta)$ are classified correctly. 


\paragraph{Error Bounds}

\begin{figure}[t]
	\centering
	\includegraphics[width=0.7\linewidth]{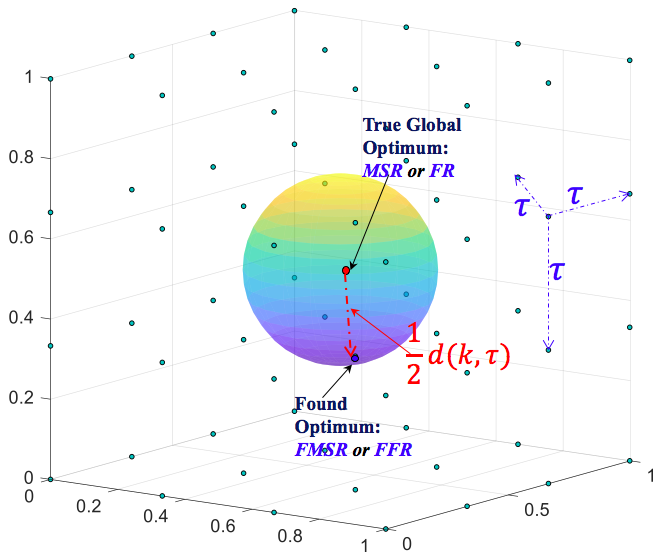}
	\caption{\response{Provable guarantees for the $\maximumsaferadius$ and $\featurerobustness_\setoffeatures$ problems on a dense $\tau$-grid (green dots) that is reached upon convergence.
	In the worst case, the true optimum (the red dot) lies in the middle between two hyper-points of the $\tau$-grid, with the distance of at most $\frac{1}{2} d(k,\tau)$ from the found optimum.} 
	}
	\label{fig:guarantee}
\end{figure}

We now bound the error of using $\finitemaximumsaferadius(\tau,k,d,\inputImage,c)$ to estimate $\maximumsaferadius(k,d,\inputImage,c)$ in $\frac{1}{2}d(k,\tau)$\response{, as illustrated in Figure~\ref{fig:guarantee}}.
First of all, we have the following lemma. Recall from Lemma~\ref{lemma:onedirection} that we already have $\maximumsaferadius(k,d,\inputImage,c) \leq \finitemaximumsaferadius(\tau,k,d,\inputImage,c)$.

\begin{lemma}\label{thm:misclassification}
If all $\tau$-grid inputs are misclassification aggregators with respect to $\frac{1}{2}d(k,\tau)$, then 
$\maximumsaferadius(k,d,\inputImage,c) \geq  \finitemaximumsaferadius(\tau,k,d,\inputImage,c) - \frac{1}{2}d(k,\tau)$.
\end{lemma}
\begin{proof} 
%
We prove by contradiction. Assume that $\finitemaximumsaferadius(\tau,k,d,\inputImage,c) = d'$ for some $d' >0$, and $\maximumsaferadius(k,d,\inputImage,c) <  d' - \frac{1}{2}d(k,\tau)$. Then there must exist an input $\inputImage'$ such that $\inputImage' \in adv_{k,d}(\inputImage,c)$ and  
\begin{equation}\label{equ:lemma41}
\distance{k}{\inputImage'-\inputImage} = \maximumsaferadius(k,d,\inputImage,c)  <  d' - \frac{1}{2}d(k,\tau),
\end{equation} and $\inputImage'$ is not a $\tau$-grid input. 
By Lemma~\ref{lemma:cover}, there must exist a $\tau$-grid input $\inputImage''$ such that $\inputImage' \in \eta(\inputImage'',L_k,\frac{1}{2}d(k,\tau))$. Now 
because all $\tau$-grid inputs are misclassification aggregators with respect to $\frac{1}{2}d(k,\tau)$, we have $\inputImage'' \in adv_{k,d}(\inputImage,c)$.

By $\inputImage'' \in adv_{k,d}(\inputImage,c)$ and the fact that $\inputImage''$ is a $\tau$-grid input, we have that  
\begin{equation}\label{equ:lemma42}
\finitemaximumsaferadius(\tau,k,d,\inputImage,c) \leq \distance{k}{\inputImage - \inputImage''} \leq \distance{k}{\inputImage - \inputImage'} + \frac{1}{2}d(k,\tau) .
\end{equation}

Now, combining Equations (\ref{equ:lemma41}) and (\ref{equ:lemma42}), we have that $\finitemaximumsaferadius(\tau,k,d,\inputImage,c) < d'$, which contradicts the hypothesis that $\finitemaximumsaferadius(\tau,k,d,\inputImage,c) = d'$.
\end{proof}

\noindent

In the following, we discuss how to determine the largest $\tau$ for a Lipschitz network in order to satisfy the condition in Lemma \ref{thm:misclassification} that all $\tau$-grid inputs are misclassification aggregators with respect to $\frac{1}{2}d(k,\tau)$.

\begin{definition}\label{def:minimumconfidencegap}
Given a class label $c$, we let  
\begin{equation}
g(\inputImage',c) = \min_{c'\in C, c'\neq c} \{N(\inputImage', c) - N(\inputImage',c')\}
\end{equation}
be a function maintaining for an input $\inputImage'$ the \emph{minimum confidence margin} between the class $c$ and another class $c' \neq N(\inputImage')$.
\end{definition}
 Note that, given an input $\inputImage'$ and a class $c$, we can compute $g(\inputImage',c)$ in constant time. The following lemma shows that the above-mentioned condition about misclassification aggregators  can be obtained if  $\tau$ is sufficiently small. 


\begin{lemma}\label{lemma:twodirection}
Let $N$ be a Lipschitz network with a Lipschitz constant $\lipschitzConstant_c$ for every class $c\in C$. If 
\begin{equation}
d(k,\tau) \leq  \frac{2g(\inputImage',N(\inputImage'))}{\max_{c\in C, c\neq N(\inputImage')} (\lipschitzConstant_{N(\inputImage')}+\lipschitzConstant_{c})}
\end{equation} for all $\tau$-grid input $\inputImage'\in \tauimage(\inputImage,k,d)$, then all $\tau$-grid inputs are misclassification aggregators with respect to $\frac{1}{2}d(k,\tau)$.  
\end{lemma}
\begin{proof}
For any input $\inputImage''$ whose closest $\tau$-grid input is $\inputImage'$, we have
\begin{equation}
\response{
    \begin{split}
        & g(\inputImage',N(\inputImage')) - g(\inputImage'',N(\inputImage')) \\
        = & \min_{c\in C, c\neq N(\inputImage')} \{N(\inputImage',N(\inputImage')) - N(\inputImage',c)\} - \min_{c'\in C, c'\neq N(\inputImage')} \{N(\inputImage'',N(\inputImage')) - N(\inputImage'',c')\} \\
        \leq & \max_{c'\in C, c'\neq N(\inputImage')} \{N(\inputImage',N(\inputImage')) -  N(\inputImage',c') - N(\inputImage'',N(\inputImage')) + N(\inputImage'',c')\} \\
        \leq & \max_{c'\in C, c'\neq N(\inputImage')} \{|N(\inputImage',N(\inputImage')) - N(\inputImage'',N(\inputImage'))| + |N(\inputImage'',c') - N(\inputImage',c')|\} \\
        \leq & \max_{c'\in C, c'\neq N(\inputImage')} (\lipschitzConstant_{N(\inputImage')}+\lipschitzConstant_{c'})\distance{k}{\inputImage'-\inputImage''} \\
        \leq & \max_{c'\in C, c'\neq N(\inputImage')} (\lipschitzConstant_{N(\inputImage')}+\lipschitzConstant_{c'})\frac{1}{2}d(k,\tau)
    \end{split}}
\end{equation}
\commentout{
\begin{equation}
\begin{array}{ll}
& g(\inputImage',N(\inputImage')) - g(\inputImage'',N(\inputImage'))   \\
= & \displaystyle \min_{c\in C, c\neq N(\inputImage')} \{N(\inputImage',N(\inputImage')) - N(\inputImage',c)\} - \min_{c'\in C, c'\neq N(\inputImage')} \{N(\inputImage'',N(\inputImage')) - N(\inputImage'',c')\} \\
\leq & \displaystyle \max_{c'\in C, d\neq N(\inputImage')} \{N(\inputImage',N(\inputImage')) -  N(\inputImage',c') - N(\inputImage'',N(\inputImage')) + N(\inputImage'',c')\} \\ 
\leq & \displaystyle \max_{c'\in C, d\neq N(\inputImage')} \{|N(\inputImage',N(\inputImage')) - N(\inputImage'',N(\inputImage'))| + |N(\inputImage'',c') -  N(\inputImage',c')|\} \\
\leq & \max_{c'\in C, c'\neq N(\inputImage')} (\lipschitzConstant_{N(\inputImage')}+\lipschitzConstant_{c'})\distance{k}{\inputImage'-\inputImage''}\\
\leq & \max_{c'\in C, c'\neq N(\inputImage')} (\lipschitzConstant_{N(\inputImage')}+\lipschitzConstant_{d})\frac{1}{2}d(k,\tau)
\end{array}
\end{equation}
}
Now, to ensure that no class change occurs between $\inputImage''$ and $\inputImage'$, we need to have  $g(\inputImage'',N(\inputImage')) \geq 0$, which means that  $g(\inputImage',N(\inputImage')) - g(\inputImage'',N(\inputImage')) \leq g(\inputImage',N(\inputImage'))$. Therefore, we can let 
\begin{equation}
\max_{c'\in C, c'\neq N(\inputImage')} (\lipschitzConstant_{N(\inputImage')}+\lipschitzConstant_{c'})\frac{1}{2}d(k,\tau) \leq g(\inputImage',N(\inputImage')).
\end{equation}
Note that $g(\inputImage',N(\inputImage'))$ is dependent on the $\tau$-grid input $\inputImage'$, and thus can be computed when we construct the grid. Finally, we let 
\begin{equation}
d(k,\tau) \leq \frac{2g(\inputImage',N(\inputImage'))}{\max_{c'\in C, c'\neq N(\inputImage')} (\lipschitzConstant_{N(\inputImage')}+\lipschitzConstant_{c'})}
\end{equation}
Therefore, if we have the above inequality for every $\tau$-grid input, then we can conclude $g(\inputImage'',N(\inputImage')) \geq 0$ for any $\inputImage''  \in \eta(\inputImage',k,d)$, i.e., $N(\inputImage'',N(\inputImage')) \geq  N(\inputImage'',c)$ for all $c\in C$. The latter means that no class change occurs. 
\end{proof}


Combining Lemmas~\ref{lemma:onedirection}, \ref{thm:misclassification}, and \ref{lemma:twodirection}, we have the following theorem which shows that the reduction has a provable guarantee, dependent on the choice of the manipulation magnitude. 

\begin{theorem}\label{thm:lipschitzreduction}
Let $N$ be a Lipschitz network with a Lipschitz constant $\lipschitzConstant_c$ for every class $c\in C$. If $d(k,\tau) \leq \frac{2g(\inputImage', N(\inputImage'))}{\max_{c'\in C, c'\neq N(\inputImage')} (\lipschitzConstant_{N(\inputImage')}+\lipschitzConstant_{c'})}$  for all $\tau$-grid inputs $\inputImage'\in \tauimage(\inputImage,k,d)$, then we can use $\finitemaximumsaferadius(\tau,k,d,\inputImage,c)$ to estimate $\maximumsaferadius(k,d,\inputImage,c)$ with an error bound $\frac{1}{2}d(k,\tau)$.
\end{theorem}
\begin{proof}
By Lemma~\ref{lemma:onedirection}, we have $\maximumsaferadius(k,d,\inputImage,c) \leq \finitemaximumsaferadius(\tau,k,d,\inputImage,c)$ for any $\tau>0$. By Lemma \ref{thm:misclassification} and Lemma~\ref{lemma:twodirection}, when $\finitemaximumsaferadius(\tau,k,d,\inputImage,c) = d'$, we have $\maximumsaferadius(k,d,\inputImage,c) \geq d' - \frac{1}{2}d(k,\tau)$, under the condition that $d(k,\tau) \leq \frac{2g(\inputImage',N(\inputImage'))}{\max_{c'\in C, c'\neq N(\inputImage')} (\lipschitzConstant_{N(\inputImage')}+\lipschitzConstant_{c'})}$  for all $\tau$-grid inputs $\inputImage'\in \tauimage(\inputImage,k,d)$. Therefore, when $d(k,\tau) \leq \frac{2g(\inputImage',N(\inputImage'))}{\max_{c'\in C, c'\neq N(\inputImage')} (\lipschitzConstant_{N(\inputImage')}+\lipschitzConstant_{c'})}$  for all $\tau$-grid inputs $\inputImage'\in \tauimage(\inputImage,k,d)$, if we use $d'$ to estimate $\maximumsaferadius(k,d,\inputImage,c)$, we will have $d' - \frac{1}{2}d(k,\tau)\leq \maximumsaferadius(k,d,\inputImage,c)\leq d'$, i.e., the error bound is $\frac{1}{2}d(k,\tau)$. 
\end{proof}

\subsection{The Feature Robustness Problem}

\begin{figure}[t]
	\centering
	\includegraphics[width=1\linewidth]{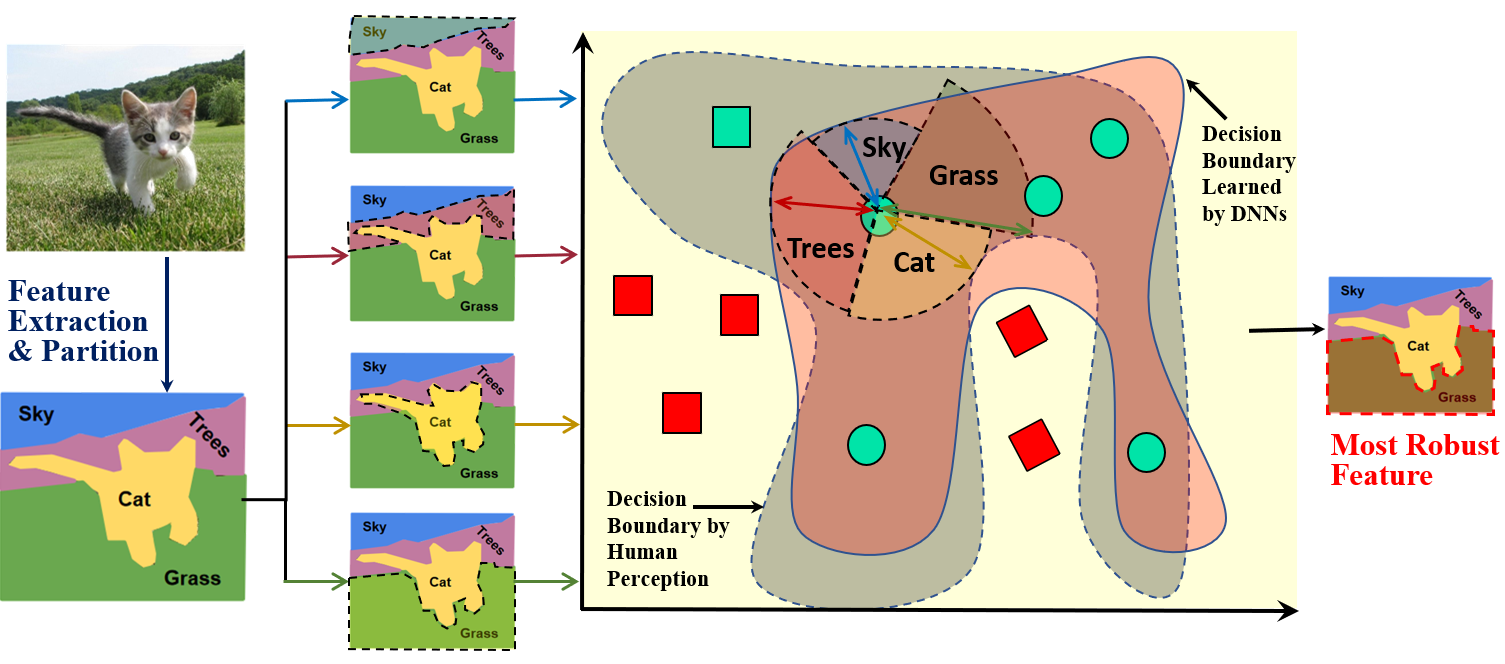}
	\caption{Illustration of the {\em feature robustness} ($\featurerobustness_\setoffeatures$) problem, which aims to find, on an original image $\inputImage$, a feature, or a subset of features, that is the most robust against adversarial perturbations. Given a benign image, we first apply feature extraction or semantic partitioning methods to produce a set of disjoint features 
	(`Sky', `Trees', `Cat', etc.), then 
	we find a set of robust features that is most resilient to adversarial perturbations (`Grass' in the figure), which quantifies the most robust direction in a safe norm ball. 
	}
	\label{fig:FeatureRobustness}
\end{figure}

The second problem studied in this paper concerns which features are the most robust against perturbations, illustrated in Figure~\ref{fig:FeatureRobustness}. The \emph{feature robustness} problem has been studied in explainable AI. For example, \cite{LIME} explains the decision making of an image classification network through different contributions of the superpixels (i.e., features) and \cite{LL2017} presents a general additive model for explaining the decisions of a network by the Shapley value computed over the set of features.  

Let $P_0(\inputImage_1,\inputImage_2) \subseteq P_0$ be the set of input dimensions on which $\inputImage_1$ and $\inputImage_2$ have different values. 
%
\begin{definition}[Feature Robustness] \label{def:featurerobustness}
The \emph{feature robustness} problem is defined as follows. 
\begin{equation} \label{equ:fsobjective}
    \featurerobustness_{\setoffeatures}(k,d,\inputImage,c) = \max_{\feature \in \setoffeatures(\inputImage)} \{ \xfeaturerobustness_{\setoffeatures}(\feature,k,d,\inputImage,c) \}
\end{equation}
where $\xfeaturerobustness_{\setoffeatures}(\feature_m,k,d,\inputImage_m,c) = $
\begin{equation} \label{equ:fsobjective2}
    \begin{cases}
        \displaystyle \min_{\inputImage_{m+1}\in \eta(\inputImage,L_k,d)} \{\distance{k}{\inputImage_m-\inputImage_{m+1}} + \featurerobustness_{\setoffeatures}(k,d,\inputImage_{m+1},c) \mid \\
        \hspace{4cm} \emptyset \neq P_0(\inputImage_m, \inputImage_{m+1}) \subseteq P_{\feature_m} \},
        & \text{if } \inputImage_m \notin adv_{k,d}(\inputImage,c) \\
        0, 
        & \text{otherwise}
    \end{cases}
\commentout{
\left\{
\begin{array}{l}
\displaystyle\min_{\inputImage_{m+1}\in \eta(\inputImage,L_k,d)} \{\distance{k}{\inputImage_m-\inputImage_{m+1}} + \featurerobustness_{\setoffeatures}(k,d,\inputImage_{m+1},c)~|~\\
\hspace{3cm}\emptyset \neq P_0(\inputImage_m, \inputImage_{m+1}) \subseteq P_{\feature_m}\},  
\text{if } \inputImage_m \notin adv_{k,d}(\inputImage,c) \\
0,   
\hfill\text{otherwise}
\end{array}
\right.}
%
\commentout{
\begin{array}{l}
\featurerobustness_{\setoffeatures}(k,d,\inputImage,c) = \\ \max_{\setoffeatures'\subseteq \setoffeatures(\inputImage), \setoffeatures'\neq \emptyset} \min_{\inputImage'\in\inputdomain} \{ \distance{k}{\inputImage - \inputImage'}~|~
  P_0(\inputImage,\inputImage') \subseteq  \setoffeatures',\inputImage' \in adv_{k,d}(\inputImage,c)\}
\end{array}}
\end{equation}
where $\setoffeatures$ is a feature extraction function, and $\inputImage_{m}, \inputImage_{m+1} (m \in \mathbb{N})$ are the  inputs before and after the application of  some manipulation on a feature $\feature_m$, respectively. If after selecting a feature $\feature_m$ no adversarial example can be reached, i.e., $\forall \inputImage_{m+1}:P_0(\inputImage_m, \inputImage_{m+1}) \subseteq P_{\feature_m} \Rightarrow \inputImage_{m+1} \notin adv_{k,d}(\inputImage,c)$, then we let $\xfeaturerobustness_{\setoffeatures}(\feature_m,k,d,\inputImage_m,c) = \depsilon$.
\end{definition}

Intuitively, the search for the most robust feature alternates between maximising over the features and minimising over the possible input dimensions within the selected feature, with the distance to the adversarial example as the objective. 
Starting from $\featurerobustness_{\setoffeatures}(k,d,\inputImage_0,c)$ where $\inputImage_0$ is the original image, the process moves to $\featurerobustness_{\setoffeatures}(k,d,\inputImage_1,c)$ by a max-min alternation on selecting feature $\feature_0$ and next input $\inputImage_1$. This continues until either an adversarial example is found, or the next input $\inputImage_i$ for some $i>0$ is outside the $d$-neighbourhood $\eta(\inputImage,L_k,d)$.  The value $\depsilon$ is  used is 
to differentiate from the case where the minimal adversarial example has exactly distance $d$ from $\inputImage_0$ and the manipulations are within $\feature_0$. \commentout{ i.e., 
\begin{equation}
\begin{array}{l}
\exists \inputImage_{1}: (P_0(\inputImage_0, \inputImage_{1}) \subseteq P_{\feature_0} \land \inputImage_{1} \in adv_{k,d}(\inputImage,c) \land \distance{k}{\inputImage_{1} - \inputImage_{0}} = d) \land\\
\forall \inputImage_{1}': (P_0(\inputImage_0, \inputImage_{1}') \subseteq P_{\feature_0} \land \inputImage_{1}' \in adv_{k,d}(\inputImage,c) \Rightarrow \distance{k}{\inputImage_{1}' - \inputImage_{0}} = d)
\end{array}
\end{equation} }
In such a case, according to Equation (\ref{equ:fsobjective2}), we have $\xfeaturerobustness_{\setoffeatures}(\feature_0,k,d,\inputImage_0,c) = d$.

Assuming  $\featurerobustness_{\setoffeatures}(k,d,\inputImage,c)$ has been computed and a \emph{distance budget} $d'\leq d$ is given to manipulate the input $\inputImage$, the following cases can be considered.
\begin{itemize}
    \item If $\featurerobustness_{\setoffeatures}(k,d,\inputImage,c) > d$, then there are robust features, and if manipulations are restricted to those features no adversarial example is possible.
    \item If $\featurerobustness_{\setoffeatures}(k,d,\inputImage,c)\leq d' \leq d$, then, no matter how one restricts the features to be manipulated, an adversarial example can be found within the budget. 
    \item If $\maximumsaferadius(k,d,\inputImage,c) \leq d' < \featurerobustness_{\setoffeatures}(k,d,\inputImage,c)\leq d$, then the existence of adversarial examples is \response{controllable}, i.e., we can choose a set of features on which the given distance budget $d'$ is insufficient to find an adversarial example. \response{
    This differs from the first case in that an adversarial example can be found if given a larger budget $d$.} 
\end{itemize} 

Therefore, studying the feature robustness problem enables a better understanding of the robustness of individual features and how the features contribute to the robustness of an image. 

It is straightforward to show that 
\begin{equation}
\maximumsaferadius(k,d,\inputImage,c) \leq \featurerobustness_{\setoffeatures}(k,d,\inputImage,c). 
\end{equation}
Compared to the absolute safety radius by  $\maximumsaferadius(k,d,\inputImage,c)$, $\featurerobustness_{\setoffeatures}(k,d,\inputImage,c) $ can be seen as a \emph{relative} safety radius, within which the existence of adversarial examples can be controlled. 
Theoretically, the $\maximumsaferadius(k,d,\inputImage,c)$ problem can be seen as a special case of the $\featurerobustness_{\setoffeatures}(k,d,\inputImage,c) $ problem, when we let $|\setoffeatures(\inputImage)|=1$. We study them separately, because the $\maximumsaferadius(k,d,\inputImage,c)$ problem is interesting on its own, and, more importantly, we show later that they can be solved using different methods. 

One can also consider a simpler variant of this problem, which aims to find a subset of features that are most resilient to perturbations, and which can be solved by only considering singleton sets of features. We omit the formalisation for reasons of space.

We illustrate the two problems, the \emph{maximum safe radius} ($\maximumsaferadius$) and (the simpler variant of)  \emph{feature robustness} ($\featurerobustness_\setoffeatures'$), through Example~\ref{exp:problems}. 

\begin{figure}[t]
	\centering
	\includegraphics[width=1\linewidth]{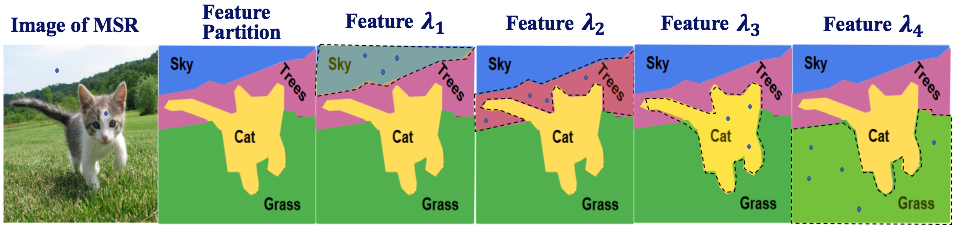}
	\caption{\response{Illustration of the \emph{maximum safe radius} ($\maximumsaferadius$) and \emph{feature robustness} ($\featurerobustness_\setoffeatures$) problems. From left to right: an adversarial example with two pixel changes, feature extraction of the image, adversarial examples with three changed pixels on features `Sky' and  `Cat', four changed pixels on `Trees', and five pixel manipulations on `Grass', respectively.}}
	\label{fig:problem_explain}
\end{figure}

\begin{example}\label{exp:problems}
As shown in \response{Figure~\ref{fig:problem_explain}}, the minimum distance from the original image to an adversary is two pixels, i.e., $\maximumsaferadius=2$ (for simplicity here we take the $L_0$-norm). That is, for a norm ball with radius less than 2, the image is absolutely safe. Note that, for $\maximumsaferadius$, the manipulations can 
span different features. After feature extraction, 
we find the \emph{maximum safe radius} of each feature, i.e., $\maximumsaferadius_{\lambda_1}=3$, $ \maximumsaferadius_{\lambda_2}=4$, $ \maximumsaferadius_{\lambda_3}=3$, $ \maximumsaferadius_{\lambda_4}=5$.

Assume that we have a norm ball of radius $d$, and a distance budget $d'$, then: 
\begin{itemize}
    \setlength\itemsep{0em}
    \item if $d=4$, then by definition we have $\featurerobustness_\setoffeatures'=4^\epsilon$, i.e., 
    manipulating `Grass'  
    cannot change the classification;
    \item if $d=10$ and $d'=7$ then we have $\featurerobustness_\setoffeatures' = 5 < d' < d$, i.e., all the features are fragile; 
    \item if $d=10$ and $d'=4$ then $d' < \featurerobustness_\setoffeatures' = 5  < d$, i.e., the existence of an adversary is controllable by restricting perturbations to `Grass'. 
\end{itemize}

\end{example}

\paragraph{Approximation Based on Finite Optimisation}
Similarly to the case of the maximum safe radius, we reduce the feature robustness problem to finite optimisation by implementing the search for adversarial examples using input manipulations.

\begin{definition}\label{def:FFR}
Let $\tau \in (0,1]$ be a manipulation magnitude.
The finite feature robustness problem $\finitefeaturerobustness_{\setoffeatures}(\tau,k,d,\inputImage,c)$ based on input manipulation is as follows:
\begin{equation}\label{equ:FFR}
\finitefeaturerobustness_{\setoffeatures}(\tau,k,d,\inputImage,c) = 
\max_{\feature \in \setoffeatures(\inputImage)} \{
\xfinitefeaturerobustness_{\setoffeatures}(\feature,\tau,k,d,\inputImage,c)
\}
\end{equation}
where $\xfinitefeaturerobustness_{\setoffeatures}(\feature_m,\tau,k,d,\inputImage_m,c) = $
\response{
\begin{equation}
    \begin{cases}
        \displaystyle\min_{X\subseteq P_{\feature_m}}\min_{\instruction\in \instructionset} \{\distance{k}{\inputImage_m-\manipulation_{\tau,X,\instruction}(\inputImage_m)} + \finitefeaturerobustness_{\setoffeatures}(k,d,\manipulation_{\tau,X,\instruction}(\inputImage_m),c)\},
        & \text{if } \inputImage_m \notin adv_{k,d}(\inputImage,c) \\
        0,
        & \text{otherwise}
    \end{cases}
\commentout{
\left\{
\begin{array}{ll}
\displaystyle\min_{X\subseteq P_{\feature_m}}\min_{\instruction\in \instructionset} \{\distance{k}{\inputImage_m-\manipulation_{\tau,X,\instruction}(\inputImage_m)} + \finitefeaturerobustness_{\setoffeatures}(k,d,\manipulation_{\tau,X,\instruction}(\inputImage_m),c)\} & \text{if } \inputImage_m \notin adv_{k,d}(\inputImage,c) \\
0   & 
\text{otherwise}
\end{array}
\right.}
\end{equation}
where $\setoffeatures$ is a feature extraction function, and $\inputImage_{m}, \manipulation_{\tau,X,\instruction}(\inputImage_m), m \in \mathbb{N},$ are the perturbed inputs before and after the application of manipulation $\manipulation_{\tau,X,\instruction}$ on a feature $\feature_m$, respectively.
If after selecting a feature $\feature_m$ no adversarial example can be reached, i.e., $\forall X\subseteq P_{\feature_m}\forall \instruction \in \instructionset: \manipulation_{\tau,X,\instruction}(\inputImage_m) \notin adv_{k,d}(\inputImage,c)$, then we let $\xfinitefeaturerobustness_{\setoffeatures}(\feature_m,\tau,k,d,\inputImage_m,c) = \depsilon$.
}
\end{definition}

Compared to Definition~\ref{def:featurerobustness}, the search for another input by $\min_{\inputImage_{m+1}\in \eta(\inputImage,L_k,d)}$ is implemented by combinatorial
search over the finite sets of feature sets and instructions. 
\commentout{
Moreover, the finite one-alternation feature robustness problem $\finitefeaturerobustness_{\setoffeatures}'(\tau,k,d,\inputImage,c)$ is as follows. 
\begin{equation}\label{equ:rfobjective2}
\max_{\setoffeatures'\subseteq \setoffeatures(\inputImage)} \min_{X\subseteq P_{\setoffeatures'}}\min_{\instruction\in {\cal I}} \{ \distance{k}{\inputImage - \manipulation_{\tau,X,i}(\inputImage)}  ~|~ \manipulation_{\tau,X,i}(\inputImage) \in adv_{k,d}(\inputImage,c)\}
\end{equation}
Moreover, for a given $\setoffeatures'\subseteq \setoffeatures(\inputImage)$, if for all $X\subseteq P_{\setoffeatures'}$ and $\instruction\in {\cal I}$ we have $\manipulation_{\tau,X,i}(\inputImage)\notin adv_{k,d}(\inputImage,c)$ then we let  $\min_{X\subseteq P_{\setoffeatures'}}\min_{\instruction\in {\cal I}} \{ \distance{k}{\inputImage - \manipulation_{\tau,X,i}(\inputImage)}  ~|~ \manipulation_{\tau,X,i}(\inputImage) \in adv_{k,d}(\inputImage,c)\} = \depsilon$. Comparing with Equation (\ref{equ:simplefeaturerobustness}), after determining a set $\setoffeatures'$ of features by $\max_{\setoffeatures'\subseteq \setoffeatures(\inputImage)}$, the finding of an input by $\maximumsaferadius$  is implemented by discrete operations
$\min_{X\subseteq P_{\setoffeatures'}}\min_{\instruction\in {\cal I}}$. 
}

\paragraph{Error Bounds}
The case for the feature robustness problem largely follows that of the maximum safe radius problem. First of all, we have the following lemma which bounds the error of $\finitefeaturerobustness_{\setoffeatures}(\tau,k,d,\inputImage,c)$ to $\frac{1}{2}d(k,\tau)$, which depends on the value of magnitude.

\begin{lemma}\label{thm:misclassification2}
If all $\tau$-grid inputs are misclassification aggregators with respect to $\frac{1}{2}d(k,\tau)$, then 
$\featurerobustness_{\setoffeatures}(k,d,\inputImage,c) \geq  \finitefeaturerobustness_{\setoffeatures}(\tau,k,d,\inputImage,c) - \frac{1}{2}d(k,\tau)$.
\end{lemma}
\begin{proof} 
%
\response{
We prove by contradiction. Assume that $\finitefeaturerobustness_{\setoffeatures}(\tau,k,d,\inputImage,c) = d'$ for some $d' >0$, and $\featurerobustness_{\setoffeatures}(k,d,\inputImage,c) <  d' - \frac{1}{2}d(k,\tau)$. Then, for all subsets $\setoffeatures\subseteq \setoffeatures(\inputImage)$ of features, either for all $X\subseteq \bigcup_{\feature\in \setoffeatures}P_\feature$ and $\instruction \in \instructionset$ we have $\manipulation_{\tau,X,\instruction}(\inputImage)\notin adv_{k,d}(\inputImage,c)$, or 
there must exist $X\subseteq \bigcup_{\feature\in \setoffeatures}P_\feature$ and $\instruction \in \instructionset$ such that  
\begin{equation}\label{equ:lemma61}
\inputImage' = \manipulation_{\tau,X,\instruction}(\inputImage)  \in adv_{k,d}(\inputImage,c) \text{ and }\distance{k}{\inputImage'-\inputImage} <  d' - \frac{1}{2}d(k,\tau),
\end{equation} and $\inputImage'$ is not a $\tau$-grid input. 
}

For the latter case, by Lemma~\ref{lemma:cover}, there must exist a  $\tau$-grid input $\inputImage''$ such that $\inputImage' \in \eta(\inputImage'',L_k,\frac{1}{2}d(k,\tau))$. Now 
because all $\tau$-grid inputs are misclassification aggregators with respect to $\frac{1}{2}d(k,\tau)$, we have $\inputImage'' \in adv_{k,d}(\inputImage,c)$. 
By $\inputImage'' \in adv_{k,d}(\inputImage,c)$ and the fact that $\inputImage''$ is a $\tau$-grid input, we have that  
\begin{equation}\label{equ:lemma62}
\distance{k}{\inputImage - \inputImage''} \leq \distance{k}{\inputImage - \inputImage'} + \frac{1}{2}d(k,\tau) .
\end{equation}
Therefore, we have $\finitefeaturerobustness_{\setoffeatures}(\tau,k,d,\inputImage,c) < d'$ by the  combining Equations (\ref{equ:lemma61}) and (\ref{equ:lemma62}). However, this contradicts the hypothesis that $\finitefeaturerobustness_{\setoffeatures}(\tau,k,d,\inputImage,c) = d'$.

For the former case, we have $\finitefeaturerobustness_{\setoffeatures}(\tau,k,d,\inputImage,c) = d' > d$. If $\featurerobustness_{\setoffeatures}(k,d,\inputImage,c)< d' - \frac{1}{2}d(k,\tau)$, then there exists an $\inputImage'$ such that $\inputImage' \in \eta(\inputImage'',L_k,\frac{1}{2}d(k,\tau))$ for some  $\tau$-grid input $\inputImage''$. By the definition of misclassification aggregator, we have $\inputImage'' \in adv_{k,d}(\inputImage,c)$. This contradicts the hypothesis that $\finitefeaturerobustness_{\setoffeatures}(\tau,k,d,\inputImage,c) = d' > d$.
\end{proof}

Combining Lemmas~\ref{lemma:onedirection}, \ref{lemma:twodirection}, and \ref{thm:misclassification2}, we have the following theorem which shows that the reduction has a provable guarantee under the assumption of Lipschitz continuity.
The approximation error depends
linearly on the prediction confidence on discretised `grid' inputs and is inversely proportional with respect to the Lipschitz constants of the network.

\response{
\begin{theorem}\label{thm:lipschitzreduction2}
Let $N$ be a Lipschitz network with a Lipschitz constant $\lipschitzConstant_c$ for every class $c\in C$. If $d(k,\tau) \leq \frac{2g(\inputImage',N(\inputImage'))}{\max_{c'\in C, c'\neq N(\inputImage')} (\lipschitzConstant_{N(\inputImage')}+\lipschitzConstant_{c'})}$  for all $\tau$-grid inputs $\inputImage'\in \tauimage(\inputImage,k,d)$, then we can use $\finitefeaturerobustness_{\setoffeatures}(\tau,k,d,\inputImage,c)$ to estimate $\featurerobustness_{\setoffeatures}(k,d,\inputImage,c)$ with an error bound $\frac{1}{2}d(k,\tau)$.
\end{theorem}
\begin{proof}
By Lemma~\ref{lemma:onedirection}, we have  $\featurerobustness_{\setoffeatures}(k,d,\inputImage,c) \leq \finitefeaturerobustness_{\setoffeatures}(\tau,k,d,\inputImage,c)$ for any $\tau>0$. By Lemma~\ref{lemma:twodirection} and Lemma~\ref{thm:misclassification2}, when $\finitefeaturerobustness_{\setoffeatures}(\tau,k,d,\inputImage,c) = d'$, we have $\featurerobustness_{\setoffeatures}(k,d,\inputImage,c) \geq d' - \frac{1}{2}d(k,\tau)$, under the condition that $d(k,\tau) \leq \frac{2g(\inputImage',N(\inputImage'))}{\max_{c'\in C, c'\neq N(\inputImage')} (\lipschitzConstant_{N(\inputImage')}+\lipschitzConstant_{c'})}$  for all $\tau$-grid inputs $\inputImage'\in \tauimage(\inputImage,k,d)$. Therefore, when $d(k,\tau) \leq \frac{2g(\inputImage',N(\inputImage'))}{\max_{c'\in C, c'\neq N(\inputImage')} (\lipschitzConstant_{N(\inputImage')}+\lipschitzConstant_{c'})}$  for all $\tau$-grid inputs $\inputImage'\in \tauimage(\inputImage,k,d)$, if we use $d'$ to estimate $\featurerobustness_{\setoffeatures}(k,d,\inputImage,c)$, we will have $d' - \frac{1}{2}d(k,\tau)\leq \featurerobustness_{\setoffeatures}(k,d,\inputImage,c)\leq d'$, i.e., the error bound is $\frac{1}{2}d(k,\tau)$. 
\end{proof}
}

\section{A Game-Based Approximate Verification Approach}\label{sec:approach}

In this section, we define a two-player game and show that the solutions of the two finite optimisation problems, $\finitemaximumsaferadius(k,d,\inputImage,c)$ and $\finitefeaturerobustness_{\setoffeatures}(k,d,\inputImage,c)$, given in Expressions (\ref{equ:lsobjective2}) and (\ref{equ:FFR}) can be reduced to the computation of the rewards of Player~$\playerOne$ taking an optimal strategy. The two problems differ in that they induce games in which the two players are \emph{cooperative} or \emph{competitive}, respectively.

%
The game proceeds by constructing a sequence of atomic input manipulations to implement the optimisation objectives in Equations (\ref{equ:lsobjective2}) and (\ref{equ:FFR}). 

\subsection{Problem Solving as a Two-Player Turn-Based Game}\label{sec:twoplayergame}


The game has two players, who take turns to act. Player~$\playerOne$ selects features and Player~$\playerTwo$ then selects an atomic input manipulation within the selected features.
While Player~$\playerTwo$ aims to minimise the distance to an adversarial  example, depending on the optimisation objective designed for either $\finitemaximumsaferadius(k,d,\inputImage,c)$ or $\finitefeaturerobustness_{\setoffeatures}(k,d,\inputImage,c)$, Player~$\playerOne$ can be cooperative or competitive. 
We remark that, in contrast to \cite{WHK2017} where the games were originally introduced, we do not consider the nature player.
%
%
%

\begin{figure}[t]
	\centering
	\includegraphics[width=1\linewidth]{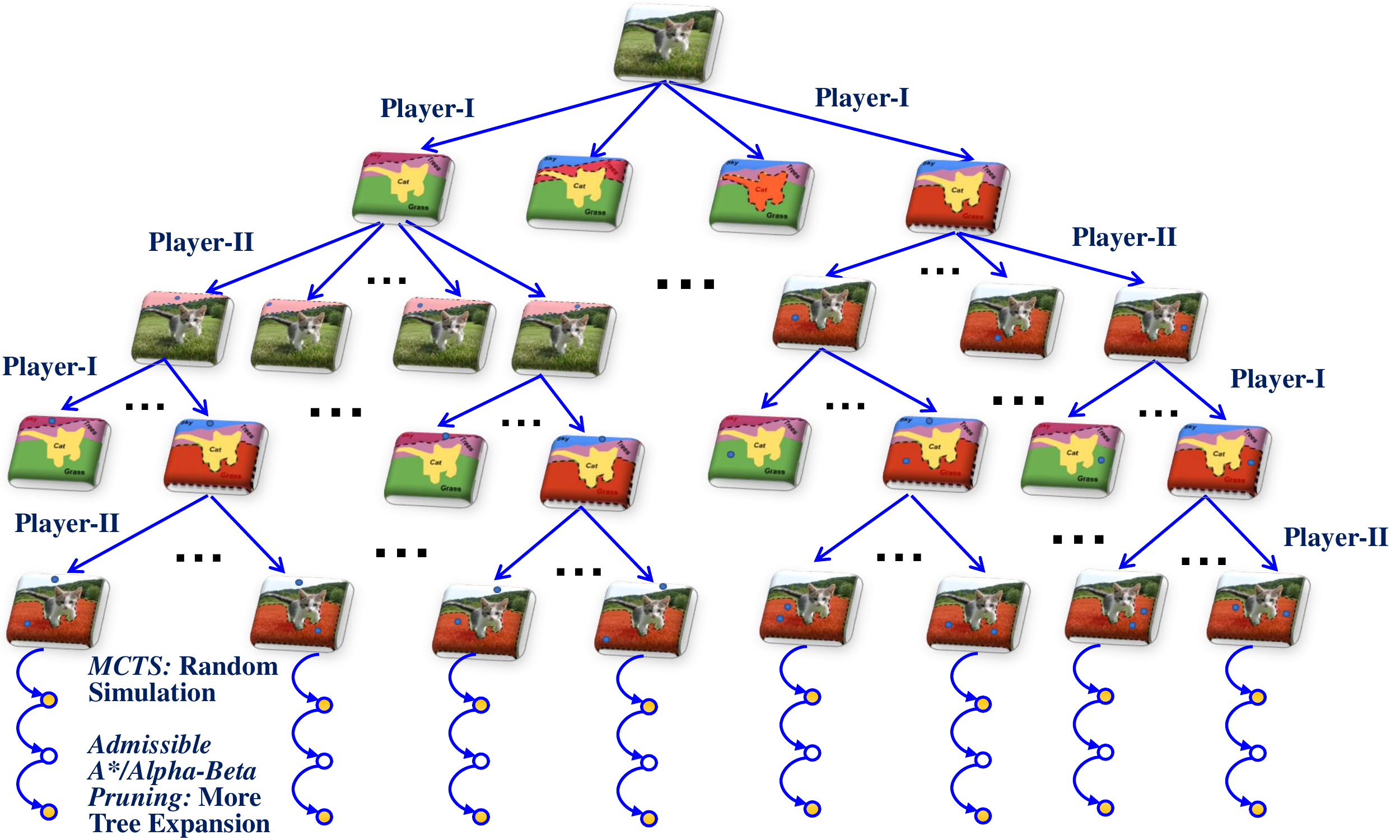}
	\caption{\response{Two-player turn-based solution for finite optimisation. 
	Player~$\playerOne$ selects features and Player~$\playerTwo$ then performs an atomic input manipulation within the selected features. For the {\em Maximum Safe Radius} problem, both Player~$\playerOne$ and Player~$\playerTwo$ aim to minimise the distance to an adversarial example; for the {\em Feature Robustness} problem, while Player~$\playerTwo$ has the same objective, Player$~\playerOne$ plays against this, i.e., aiming to prevent the reaching of adversarial examples by taking suitable actions.
	%
	The game  terminates when an adversary is found or the distance budget for adversarial perturbation has been reached.}}
	\label{fig:game}
\end{figure}

Formally, we let $M(k,d,\inputImage,c)=(S \cup (S \times \setoffeatures(\inputImage)),s_0,\{T_a\}_{a \in \{\playerOne,\playerTwo\}},L)$ be a game model, where 
\begin{itemize}
    \item $S$ is a set of game states belonging to Player~$\playerOne$ such that each state represents an input in $ \eta(\inputImage,L_k,d)$, and $S\times \setoffeatures(\inputImage)$ is a set of game states belonging to Player~$\playerTwo$ where $\setoffeatures(\inputImage)$ is a set of features of input $\inputImage$. We write $\inputImage(s)$ for the input associated to the state $s\in S$.
    
    \item $s_0\in S$ is the initial game state such that $\inputImage(s_0)$ is the original input $\inputImage$. 
    
    \item The transition relation $T_\playerOne: S \times\setoffeatures(\inputImage)\rightarrow S \times\setoffeatures(\inputImage)$ is defined as  
    \begin{equation}
        T_\playerOne(s,\feature)=(s,\feature),
    \end{equation} 
    and transition relation $T_\playerTwo: (S \times\setoffeatures(\inputImage))\times \powerset{P_0}\times \instructionset \rightarrow S$ is defined as 
    \begin{equation}\label{eqn-20}
        T_\playerTwo((s,\feature),X,\instruction)=\manipulation_{\tau,X,\instruction}(\inputImage(s)),
    \end{equation}
    where $X\subseteq P_\feature$ is a set of input dimensions within feature $\feature$, $\instruction: P_0 \rightarrow  \{-1,+1\}$ is a manipulation instruction, and $\manipulation_{\tau,X,\instruction}$ is an atomic dimension manipulation as defined in Definition~\ref{def:atomicManipulation}. Intuitively, in every game state $s \in S$, Player~$\playerOne$ will choose a feature $\feature$, and, in response to this, Player~$\playerTwo$ will choose an atomic input manipulation $\manipulation_{\tau,X,\instruction}$. 
    
    \item The labelling function $L:S\cup (S\times \setoffeatures(\inputImage))\rightarrow C$ assigns to each state $s$ or $(s,\feature)$ a class $N(\inputImage(s))$. 
\end{itemize}
\response{Figure \ref{fig:game}} illustrates the game model with a partially-expanded game tree. 
  
\paragraph{Strategy Profile}

A path (or game play) of the game model is a sequence $s_1u_1s_2u_2...$ of game states such that,  for all $k\geq 1$, we have $u_k = T_{\playerOne}(s_k,\feature_k)$ for some feature $\feature_k$ and $s_{k+1}=T_\playerTwo((s_k,\feature_k),X_k,\instruction_k)$ for some $(X_k,\instruction_k)$. Let $last(\rho)$ be the last state of a finite path $\rho$, and $Path_a^F$ be the set of finite paths such that $last(\rho)$ belongs to player $a\in \{\playerOne,\playerTwo\}$. 

\begin{definition}
A stochastic strategy $\strategy_\playerOne: Path_\playerOne^F\rightarrow \mathcal{D}(\setoffeatures(\inputImage))$ of Player~$\playerOne$ maps each finite path to a distribution over the next actions, and similarly for 
$\strategy_\playerTwo:Path_\playerTwo^F\rightarrow \mathcal{D}(\powerset{P_0}\times {\cal I})$ for  Player~$\playerTwo$.
We call $\strategy = (\strategy_\playerOne,\strategy_\playerTwo)$ a strategy profile.
\end{definition}

A strategy $\sigma$ is deterministic if $\sigma(\rho)$ is a Dirac distribution, and is memoryless if $\sigma(\rho)=\sigma(last(\rho))$ for all finite paths $\rho$. 


\paragraph{Rewards}

We define a reward $\reward(\strategy,\rho)$ for a given strategy profile $\strategy = (\strategy_\playerOne,\strategy_\playerTwo)$ and a finite path $\rho\in \bigcup_{a\in\{\playerOne,\playerTwo\}}Path_a^F$. 
The idea of the reward is to accumulate the distance to 
the adversarial example found over a path. 
Note that, given $\strategy$, the game becomes a 
deterministic 
system. 
Let $\inputImage_\rho' = \inputImage(last(\rho))$  be  the input associated with the last state of the path $\rho$. We write 
\begin{equation} \label{equ:pathterminalstate}
\terminate(\rho) \equiv (N(\inputImage_\rho')=c) \lor (\distance{k}{ \inputImage_\rho'-\inputImage} > d), 
\end{equation}
representing that the path has reached a state whose associated input either is in the target class $c$ or lies outside the region $ \eta(\inputImage,L_k,d)$. The path $\rho$ can be terminated whenever $\terminate(\rho)$ is satisfied. It is not hard to see that, due to the constraints in Definition~\ref{def:constraints}, every infinite path has a finite prefix which can be terminated
\response{(that is, either when an adversarial example is found or the distance to the original image has exceeded $d$). During each expansion of the game model, an atomic manipulation is employed, which excludes the possibility that an input dimension is perturbed in smaller and smaller steps.}

\begin{definition}\label{def:reward}
    Given a strategy profile $\strategy = (\strategy_\playerOne,\strategy_\playerTwo)$ and a finite path $\rho$, we define a reward function as follows:
    $ \reward(\strategy,\rho) = $
\response{
    \begin{equation} \label{rewardFunction}
        \begin{cases}
            \distance{k}{ \inputImage_\rho'-\inputImage},
            & \text{if } \terminate(\rho) \text{ and }\rho\in Path_\playerOne^F \\
            \displaystyle \sum_{\feature \in \setoffeatures(\inputImage)} \strategy_\playerOne(\rho)(\feature) \cdot \reward(\strategy,\rho T_\playerOne(last(\rho),\feature)), 
            & \text{if } \neg \terminate(\rho) \text{ and }\rho\in Path_\playerOne^F \\
            \displaystyle \sum_{(X,\instruction)\in \powerset{P_0}\times \instructionset} \strategy_\playerTwo(\rho )(X,\instruction) \cdot \reward(\strategy,\rho T_\playerTwo(last(\rho),X,\instruction)),
            & \text{if } \rho\in Path_\playerTwo^F
        \end{cases}
    \end{equation}  
where $\strategy_\playerOne(\rho)(\feature)$ is the probability of selecting feature $\feature$ on finite path $\rho$ by Player~$\playerOne$, and $\strategy_\playerTwo(\rho)(X,\instruction)$ is the probability of selecting atomic input manipulation $\manipulation_{\tau,X,\instruction}$ based on $\rho$ by Player~$\playerTwo$. The expression $\rho T_\playerOne(last(\rho),\feature)$ is the resulting path of Player~$\playerOne$ selecting $\feature$, and $\rho T_\playerTwo(last(\rho),X,\instruction)$ is the resulting path of Player~$\playerTwo$ applying $\manipulation_{\tau,X,\instruction}$ on $\inputImage_\rho'$. We note that a path only terminates on Player~$\playerOne$ states. 
}
\end{definition}

Intuitively, if an adversarial example is found then the reward assigned is the 
distance to the original input, otherwise it is the weighted summation of the rewards of its children.

\paragraph{Players' Objectives}

Players' strategies are to maximise their rewards in a game. The following players' objectives are designed to match the finite optimisation problems stated in Equations (\ref{equ:lsobjective2}) and (\ref{equ:FFR}).


\begin{definition}
In a game, Player~$\playerTwo$ chooses a strategy $\strategy_\playerTwo$ to minimise the reward $\reward((\strategy_\playerOne,\strategy_\playerTwo),s_0)$, whilst Player~$\playerOne$ has different goals depending on the optimisation problem under consideration. 
\begin{itemize}
    \item For the \emph{maximum safe radius} problem, Player~$\playerOne$ chooses a strategy $\strategy_\playerOne$ to minimise the reward \response{$\reward((\strategy_\playerOne,\strategy_\playerTwo),s_0)$}, based on the strategy $\strategy_\playerTwo$ of Player~$\playerTwo$. That is, the two players are \emph{cooperative}. 
    \item For the \emph{feature robustness} problem, Player~$\playerOne$ chooses a strategy $\strategy_\playerOne$ to maximise \response{$\reward((\strategy_\playerOne,\strategy_\playerTwo),s_0)$}, based on the strategy $\strategy_\playerTwo$ of Player~$\playerTwo$. That is, the two players are \emph{competitive}. 
\end{itemize}
\end{definition}

The goal of the game is for Player~$\playerOne$ to choose a strategy $\strategy_{\playerOne}$ to optimise its objective, to be formalised below.

\subsection{Safety Guarantees via Optimal Strategy}

%

For different objectives $x\in \{\maximumsaferadius(k,d,\inputImage,c),\featurerobustness_{\setoffeatures}(k,d,\inputImage,c)\}$ of Player~$\playerOne$, we construct different games. Given a game model $M(k,d,\inputImage,c)$ and an objective $x$ of Player~$\playerOne$, there exists an \emph{optimal strategy profile} $\strategy=(\strategy_\playerOne,\strategy_\playerTwo)$, obtained by both players optimising their objectives. We will consider the algorithms to compute the optimal strategy profile in Section~\ref{sec:algorithms}. Here we focus on whether the obtained optimal strategy profile $\strategy$ is able to implement the finite optimisation problems in Equations~(\ref{equ:lsobjective2}) and (\ref{equ:FFR}).

First of all, we formally define the goal of the game. 

\begin{definition}
Given a game model $M(k,d,\inputImage,c)$, an objective $x $ of Player~$\playerOne$, and an optimal strategy profile $\strategy=(\strategy_\playerOne,\strategy_\playerTwo)$, the goal of the game is to compute the value 
\begin{equation}
val(M(k,d,\inputImage,c), x) = R(\strategy, s_0)
\end{equation}
That is, the goal is to compute the reward of the initial state $s_0$ based on $\strategy$. Note that an initial state $s_0$ is also a finite path, and it is a Player~$\playerOne$ state. 
\end{definition}

We have the following Theorems~\ref{thm:lsgame} and \ref{thm:frgame} to confirm that the game can return the optimal values for the two finite optimisation problems.  
\begin{theorem} \label{thm:lsgame}
    Assume that Player~$\playerOne$ has the objective $\maximumsaferadius(k,d,\inputImage,c)$. Then 
    \begin{equation}\label{equ:valequ}
        val(M(k,d,\inputImage,c), \maximumsaferadius(k,d,\inputImage,c)) = \finitemaximumsaferadius(\tau,k,d,\inputImage,c)
    \end{equation}
\end{theorem}
\begin{proof}
    First, we show  that $ \distance{k}{\inputImage-\inputImage'} \geq val(M(k,d,\inputImage,c), \maximumsaferadius(k,d,\inputImage,c))$ for any input $\inputImage'$ such that $\inputImage' \in \eta(\inputImage,L_k,d)$, $\inputImage' \in adv_{k,d}(\inputImage,c)$, and $\inputImage'$ is a $\tau$-grid input.
    Intuitively, it says that Player~$\playerOne$ reward from the game on the initial state $s_0$ is no greater than the distance to 
    any $\tau$-grid adversarial example. That is, once computed, the $val(M(k,d,\inputImage,c), \maximumsaferadius(k,d,\inputImage,c))$ is a lower bound of the optimisation problem $\finitemaximumsaferadius(\tau,k,d,\inputImage,c)$. 
    This can be obtained by the fact that every  $\tau$-grid input can be reached by some game play. 

    Second, from the termination condition of the game plays, we can see that if $ val(M(k,d,\inputImage,c), \maximumsaferadius(k,d,\inputImage,c)) \leq \distance{k}{\inputImage-\inputImage'}$ for some $\inputImage'$ then there must exist some $\inputImage''$ such that $ val(M(k,d,\inputImage,c), \maximumsaferadius(k,d,\inputImage,c)) = \distance{k}{\inputImage-\inputImage''}$. Therefore, we have that $val(M(k,d,\inputImage,c), \maximumsaferadius(k,d,\inputImage,c))$ is the minimum value of  $\distance{k}{\inputImage-\inputImage'}$ among all $\inputImage'$ with $\inputImage' \in \eta(\inputImage,L_k,d)$, $\inputImage' \in adv_{k,d}(\inputImage,c)$, and $\inputImage'$ is a $\tau$-grid input.
    
    Finally, we notice that the above minimum value of $\distance{k}{\inputImage-\inputImage'}$ is equivalent to the optimal value required by Equation~(\ref{equ:lsobjective2}). 
\end{proof}

\begin{theorem} \label{thm:frgame}
    Assume that Player~$\playerOne$ has the objective $\featurerobustness_{\setoffeatures}(k,d,\inputImage,c)$. Then 
    \begin{equation} \label{equ:valequ}
        val(M(k,d,\inputImage,c), \featurerobustness_{\setoffeatures}(k,d,\inputImage,c)) = \finitefeaturerobustness_{\setoffeatures}(\tau,k,d,\inputImage,c)
    \end{equation}
\end{theorem}
\begin{proof}
    First of all, let $\setoffeatures_1$ be the set of features and $\manipulationset_1$ be the set of atomic input manipulations in achieving the optimal value of $\finitefeaturerobustness_{\setoffeatures}(\tau,k,d,\inputImage,c)$. We can construct a game play for $(\setoffeatures_1,\manipulationset_1)$. More specifically, the game play leads from the initial state to a terminal state, by \response{recursively} selecting an unused input manipulation and its associated feature and defining the corresponding moves for Player~$\playerOne$ and Player~$\playerTwo$, respectively. Therefore, because the strategy profile $\strategy$ is optimal, we have $val(M(k,d,\inputImage,c), \featurerobustness_{\setoffeatures}(k,d,\inputImage,c)) \geq \finitefeaturerobustness_{\setoffeatures}(\tau,k,d,\inputImage,c)$. 
    
    On the other hand, we notice that the ordering of the applications of atomic input manipulations does not matter, because the reward of the terminal state is the distance from its associated input to the original input. Therefore, because the game explores exactly all the possible applications of atomic input manipulations and $\finitefeaturerobustness_{\setoffeatures}(\tau,k,d,\inputImage,c)$ is the optimal value by its definition, by Lemma \ref{lemma:atomicinput} we have that $val(M(k,d,\inputImage,c), \featurerobustness_{\setoffeatures}(k,d,\inputImage,c)) \leq \finitefeaturerobustness_{\setoffeatures}(\tau,k,d,\inputImage,c)$. 
\end{proof}


Combining Theorems~\ref{thm:lsgame}, \ref{thm:frgame} with Theorems~\ref{thm:lipschitzreduction}, \ref{thm:lipschitzreduction2}, we have the following corollary, which states that the optimal game strategy is able to achieve the optimal value for the maximum safe radius problem $\maximumsaferadius(k,d,\inputImage,c)$ and the feature robustness problem $\featurerobustness_{\setoffeatures}(k,d,\inputImage,c)$ with an error bound $\frac{1}{2}d(k,\tau)$. 

\begin{corollary}
    The two-player turn-based game is able to solve the \emph{maximum safe radius} problem of Equation~(\ref{equ:lsobjective}) and the \emph{feature robustness} problem of Equation~(\ref{equ:fsobjective}) with an error bound $\frac{1}{2}d(k,\tau)$, when the magnitude $\tau$ is such that \response{$d(k,\tau) \leq \frac{2g(\inputImage',N(\inputImage'))}{\max_{c'\in C, c'\neq N(\inputImage')} (\lipschitzConstant_{N(\inputImage')}+\lipschitzConstant_{c'})}$} for all $\tau$-grid inputs $\inputImage'\in \tauimage(\inputImage,k,d)$. 
\end{corollary}

%
Furthermore, we have the following lemma. 
\begin{lemma} \label{lemma:det}
    For game $M(k,d,\inputImage,c)$ with goal $val(M(k,d,\inputImage,c),\maximumsaferadius(k,d,\inputImage,c))$, deterministic and memoryless strategies suffice for Player~$\playerOne$, and similarly for $M(k,d,\inputImage,c)$  with goal $val(M(k,d,\inputImage,c), \featurerobustness_{\setoffeatures}(k,d,\inputImage,c))$.
\end{lemma}

\subsection{Complexity of the Problem} \label{sec:complexity}
As a by-product of Lemma~\ref{lemma:det}, the theoretical complexity of the problems is in PTIME, with respect to the size of the game model $M(k,d,\inputImage,c)$. However, 
the size of the game is exponential with respect to the number of input dimensions. 
More specifically, we have the following complexity result with respect to the manipulation magnitude $\tau$, the pre-specified range size $d$, and the number of input dimensions $n$. 

\begin{theorem} \label{thm:complexity}
    Given a game $M(k,d,\inputImage,c)$, the computational time needed for the value $val(M(k,d,\inputImage,c), x)$, where $x\in \{\maximumsaferadius(k,d,\inputImage,c),\featurerobustness_{\setoffeatures}(k,d,\inputImage,c)\}$, is polynomial with respect to $d/\tau$ and exponential with respect to $n$.
\end{theorem}
\begin{proof}
    We can see that the size of the grid, measured as the number $|\tauimage(\inputImage,k,d)|$  of $\tau$-grid inputs in $\eta(\inputImage,L_k,d)$, is polynomial with respect to $d/\tau$ and exponential with respect to $n$. From a $\tau$-grid to any of its neighbouring $\tau$-grids, each player needs to take a move. Therefore, the number of game states is doubled (i.e., polynomial) over $|\tauimage(\inputImage,k,d)|$. This yields PTIME  complexity of solving the game. 
\end{proof}

Considering that the problem instances we work with usually have a large input dimensionality, this complexity suggests that directly working with the explicit game models is impractical. If we consider an alternative representation of a game tree (i.e., an unfolded game model) of finite depth to express the complexity, 
the number of nodes on the tree is $O(n^h)$ for $h$ the length of the longest finite path without a terminating state.  
While the precise size of $O(n^{h})$ is dependent on the problem (including the image $\inputImage$ and the difficulty of crafting an adversarial example), it is roughly $O(50000^{100})$ for the images used in the ImageNet competition and $O(1000^{20})$ for smaller images such as GTSRB, CIFAR10, and MNIST. This is beyond the capability of existing approaches for exact or $\epsilon$-approximate computation of probability (e.g., reduction to linear programming~\cite{SWRHKK2018}, value iteration, and policy iteration, etc.) that are used in probabilistic verification.


\section{Algorithms and Implementation} \label{sec:algorithms}
In this section we describe the implementation of the game-based approach introduced in this paper.
Figure~\ref{fig:flow} presents an overview of the reductions from the original problems to the solution of a two-player game for the case of Lipschitz networks, described in Section~\ref{sec:SafetyThms}.
Because exact computation of optimal rewards is computationally hard, 
we approximate the rewards by means of algorithms that unfold the game tree based on Monte Carlo tree search (MCTS), 
Admissible A$^*$, and Alpha-Beta Pruning.


\begin{figure}[t]
    \centering
    \includegraphics[width=\textwidth]{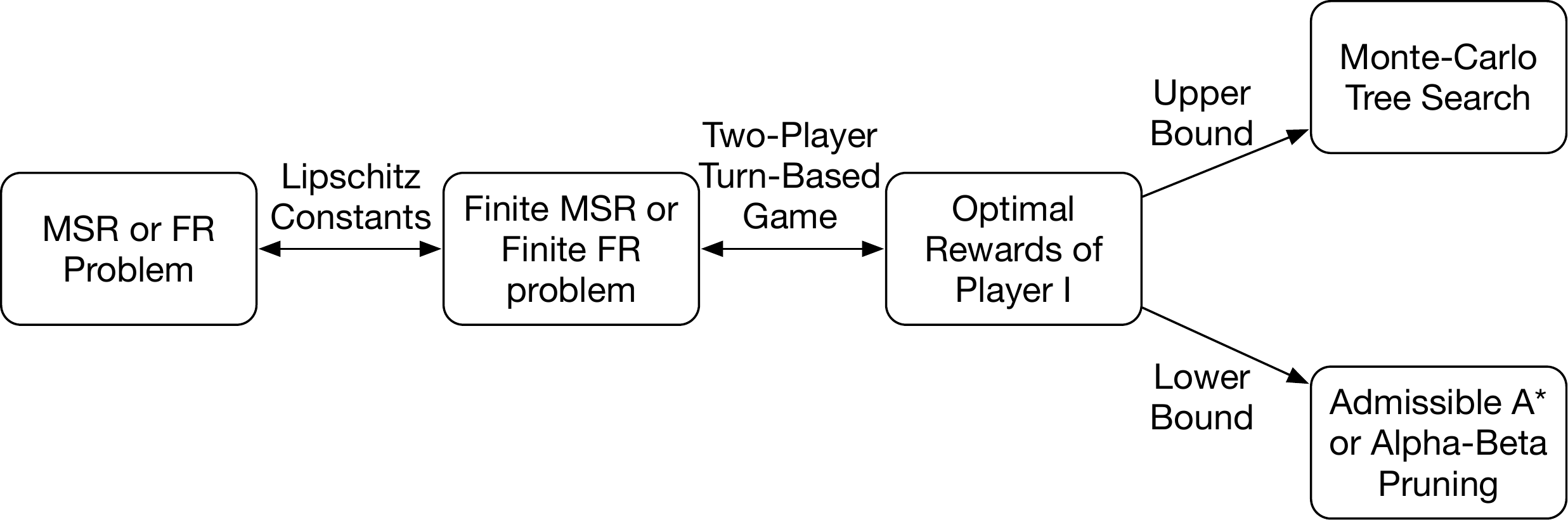}
    \caption{\response{A game-based approximate verification approach for the \emph{Maximum Safe Radius} and \emph{Feature Robustness} problems.}}
    \label{fig:flow}
\end{figure}


We take a principled approach to compute for each of the two game values,  $\maximumsaferadius(k,d,\inputImage,c)$ and $\featurerobustness_{\setoffeatures}(k,d,\inputImage,c)$, an upper bound and a lower bound. Our algorithms can gradually, but strictly, improve the bounds, so that 
they gradually converge as the computation proceeds.
For $x \in \{\maximumsaferadius(k,d,\inputImage,c),\featurerobustness_{\setoffeatures}(k,d,\inputImage,c)\}$, we write $\lowerbound_{x}$ and $\upperbound_x$  for their lower and upper bound, respectively. %
The bounds  can be interesting in their own. For example, a lower bound $\lowerbound_{\maximumsaferadius(k,d,\inputImage,c)}$ suggests absolute safety of an $L_k$ norm ball with radius $\lowerbound_{\maximumsaferadius(k,d,\inputImage,c)}$ from the original input $\inputImage$, and an upper bound $\upperbound_{\maximumsaferadius(k,d,\inputImage,c)}$ suggests  the existence of an adversarial example $\inputImage'$ such that $\distance{k}{\inputImage-\inputImage'} = \upperbound_{\maximumsaferadius(k,d,\inputImage,c)}$. \response{On the other hand, given distance budget $d'$, $\upperbound_{\featurerobustness_{\setoffeatures}(k,d,\inputImage,c)} \leq d'$ indicates an unsafe distance from which the existence of adversarial examples is not controllable.}

We consider two feature extraction approaches to generate feature partitioning: a saliency-guided $\mathsf{grey}$-box approach  and a feature-guided $\mathsf{black}$-box approach. For the $\mathsf{grey}$-box approach, adapted from \cite{RWSHKK2018}, each dimension of an input is evaluated on its sensitivity to the classification outcome of a DNN. 
%
For the $\mathsf{black}$-box procedure, Scale Invariant Feature Transform (SIFT)~\cite{SIFT} is used to extract image features, based on which a partition is computed. SIFT is a reasonable proxy for human perception, irrelevant to any image classifier, and its extracted features have a set of invariant characteristics, such as  scaling, rotation, translation, and local geometric distortion. 
Readers are referred to Section~\ref{exp:feature} for an experimental illustrations of these two approaches.

Next we present the algorithms we employ to compute the upper and lower bounds of the values of the games, as well as their convergence analysis. 

\subsection{Upper Bounds: Monte Carlo Tree Search (MCTS)}
 \label{sec:mcts}

\newcommand{\update}{U\!pdate}
\newcommand{\pathsample}{Simulation}
\newcommand{\confidence}{con\!f}
\newcommand{\normalisation}{N\!orm}
\newcommand{\selection}{selection}
\newcommand{\expansion}{expansion}
\newcommand{\backPropogation}{backPropogation}
\newcommand{\partialTree}{{\cal T}}
\newcommand{\parent}{p}
\newcommand{\children}{{\cal C}}

We present an approach based on Monte Carlo tree search (MCTS)~\cite{CWU2008} to find an optimal strategy asymptotically. 
%
%
\response{As a heuristic search algorithm for decision processes notably employed in game play, MCTS focuses on analysing the most promising moves via expanding the search tree based on random sampling of the search space.}
The 
algorithm, whose pseudo-code is presented in Algorithm~\ref{MCTS}, gradually expands a \emph{partial game tree} by sampling the strategy space of the model $M(k,d,\inputImage,c)$. With the upper confidence bound (UCB)~\cite{KS2006} as the exploration-exploitation trade-off, MCTS has a theoretical guarantee that it converges to the optimal solution when the game tree is fully explored. In the following, we explain the components of the algorithm. 

\begin{algorithm}[t]
\small
\caption{Monte Carlo Tree Search for DNN Verification}\label{MCTS}
\begin{algorithmic}[1]
\State \textbf{Input:} A game model $M(k,d,\inputImage,c)$, a termination condition $tc$
\State \textbf{Output:} $val(M(k,d,\inputImage,c), \featurerobustness_{\setoffeatures}(k,d,\inputImage,c))$ or $val(M(k,d,\inputImage,c), \maximumsaferadius(k,d,\inputImage,c))$
\Procedure{MCTS($M(k,d,\inputImage,c), tc$)}{}
\State $root \gets s_0$
\State \textbf{While}$(\neg tc)$: 
\State \quad \quad $leaf \gets \selection(root) $ 
\State \quad \quad $newnodes  \gets \expansion(M(k,d,\inputImage,c), leaf)$ 
\State \quad \quad \textbf{for} $node$ in $newnodes$: 
\State \quad \quad \quad \quad $v \gets \pathsample(M(k,d,\inputImage,c),node) $
\State \quad \quad \quad \quad $\backPropogation(node, v) $
\State \textbf{return} optimal value of the $root$ node
\EndProcedure
\end{algorithmic}
\end{algorithm}


Concerning the data structure, we maintain the set of nodes on the partial tree $\partialTree(k,d,\inputImage,c)$. 
For every node $o$ on the partial tree, we maintain three variables, $r_o$, $n_o$, $e_o$, which represent the accumulated reward, the number of visits, and the current best input with respect to the objective of the player, respectively. We remark that $e_o$ is usually different from $\inputImage(s_o)$, which is the input associated with the game state $s_o$.
Moreover, for every node $o$, we record its parent node $\parent_o$ and a set $\children_o$ of its children nodes. 
The value $val(M(k,d,\inputImage,c), x)$ of the game is approximated by 
$\distance{k}{e_{root}-\inputImage}$,
which represents the distance between the original input and the current best input maintained by the root node of the tree. 


The $\selection$ procedure starts from the $root$ node, which contains the original image, and conducts a tree traversal until reaching a $leaf$ node (Line~6). From a node, the next child node to be selected is dependent on an exploration-exploitation balance, i.e., UCB~\cite{KS2006}. More specifically, on a node $o$, for every child node $o'\in\children_o$, we let 
\begin{equation}
v(o,o') = \frac{d*n_{o'}}{r_{o'}} +  \sqrt{\frac{2\ln{n_{o}}}{n_{o'}}}
\end{equation}
be the weight of choosing $o'$ as the next node from $o$. Then the actual choice of a next node is conducted by sampling over a probabilistic distribution  $Prob_o: \children_o \rightarrow [0,1]$ such that 
\begin{equation} 
Prob_o(o') = \frac{v_{o,o'}}{\sum_{o'\in \children_o} v_{o,o'}} 
\end{equation}
which is a normalisation over the weights of all children.
%
%
On a $leaf$ node $o$, the $\expansion$ procedure returns a set of children nodes $\children_o$ by applying the transition relation in the game model $M(k,d,\inputImage,c)$ (Line~7). These new nodes are added into the partial tree $\partialTree(k,d,\inputImage,c)$. This is the only way for the partial tree to grow. 
%
%
After expanding the leaf node to have its children added to the partial tree, we call the $\pathsample$ procedure on every child node (Line~9). A simulation on a new node $o$ is a  play of the game from $o$ until it terminates. Players act randomly during the simulation. Every simulation terminates when reaching a terminal node $\inputImage'$.
%
%
Once a terminal node $\inputImage'$ is reached, a reward $\distance{k}{\inputImage-\inputImage'}$ can be computed. This reward, together with the input $\inputImage'$, is then $backpropagated$ from the new child node through its ancestors until reaching the root (Line~10). Every time a new reward $v$ is backpropagated through a node $o$, we update its associated reward $r_o$ into $r_o+v$ and increase its number of visits into $n_o+1$. The update of current best input $e_o$ depends on the player who owns the node. 
For the $\maximumsaferadius(k,d,\inputImage,c)$ game, $e_o$ is made equivalent to $e_{o'}$ such that 
\begin{equation} \label{equ:lsbestcase}
    o' = \arg\min_{o_1 \in \children_o} \distance{k}{\inputImage - e_{o_1}}
\end{equation}
For the $\featurerobustness_{\setoffeatures}(k,d,\inputImage,c)$ game, Player~$\playerTwo$ also takes the above approach, i.e., Equation~(\ref{equ:lsbestcase}), to update $e_{o}$, but for Player~$\playerOne$ we let $e_o$ be $e_{o''}$
such that 
\begin{equation} \label{equ:frbestcase}
    o'' = \arg\max_{o_1 \in \children_o} \distance{k}{\inputImage - e_{o_1}}
\end{equation}
We remark the game is not zero-sum for the maximum safe radius problem.


\subsection{Lower Bounds: Admissible A* in a Cooperative Game}

To enable the computation of lower bounds with a guarantee, we consider algorithms which can compute optimal strategy deterministically, without relying on the asymptotic convergence as MCTS does. In this section, we exploit Admissible A* to achieve the lower bound of Player~$\playerOne$  reward when it is \emph{cooperative}, i.e., $\maximumsaferadius(k,d,\inputImage,c)$, and in Section~\ref{sec:alphabetadescription} we use Alpha-Beta Pruning to obtain the lower bound of Player~$\playerOne$ reward when it is \emph{competitive}, i.e., $\featurerobustness_{\setoffeatures}(k,d,\inputImage,c)$.

The A* algorithm gradually unfolds the game model into a tree. It maintains a set of leaf nodes of the unfolded partial tree, computes an estimate for every node in the set, and selects the node with the least estimated value to expand. 
The estimation consists of two components, one for the exact cost up to now and the other for the estimated cost of reaching the goal node.  
In our case, for each game state $s$, we assign an estimated distance value 
\begin{equation}\label{equ:dist_estimation}
    distances(s) = \distance{k}{\inputImage(s)-\inputImage(s_0)} + heuristic(\inputImage(s))
\end{equation}
where the first component $\distance{k}{\inputImage(s)-\inputImage(s_0)}$ represents the distance from the initial state $s_0$ to the current state $s$, and the second component $heuristic(\inputImage(s))$ denotes the estimated distance from the current state $s$ to a terminal state.

An \emph{admissible} heuristic function is to, given a current input, never overestimate the cost of reaching the terminal game state. 
Therefore, to achieve the lower bound, we need to take an admissible heuristic function. 
We remark that, if the heuristic function is inadmissible (i.e., does not guarantee the underestimation of the cost), then the A* algorithm cannot be used to compute the lower bound, but instead can be used to compute the upper bound. 

We utilise the minimum confidence margin $g(\inputImage',N(\inputImage'))$ defined in Definition~\ref{def:minimumconfidencegap} to obtain an admissible heuristic function.

\begin{lemma}\label{lemma:heuristics}
For any game state $s$ such that $\inputImage(s)=\inputImage'$, the following heuristic function is admissible: 
\begin{equation}
    heuristic(\inputImage') =  \frac{g(\inputImage',N(\inputImage'))}{\max_{c'\in C, c'\neq N(\inputImage')} (\lipschitzConstant_{N(\inputImage')}+\lipschitzConstant_{c'})} 
\end{equation}

\end{lemma}
\begin{proof}
Consider the expression $g(\inputImage',N(\inputImage')) - g(\inputImage'',N(\inputImage'))$, where $\inputImage'$ is the current state and $\inputImage''$ is the last state before a terminal state. Then we have that 
\begin{equation}
    g(\inputImage',N(\inputImage')) - g(\inputImage'',N(\inputImage')) \leq g(\inputImage',N(\inputImage')) 
\end{equation} 
Now because 
\begin{equation}
    \begin{split}
        & g(\inputImage',N(\inputImage')) - g(\inputImage'',N(\inputImage')) \\
        = & \displaystyle \min_{c\in C, c\neq N(\inputImage')} \{N(\inputImage',N(\inputImage')) - N(\inputImage',c)\} - \min_{c'\in C, c'\neq N(\inputImage')} \{N(\inputImage'',N(\inputImage')) - N(\inputImage'',c')\} \\
        \leq & \displaystyle \max_{c'\in C, c'\neq N(\inputImage')} \{|N(\inputImage',N(\inputImage')) - N(\inputImage'',N(\inputImage'))| + |N(\inputImage'',c') -  N(\inputImage',c')|\} \\
        \leq & \displaystyle \max_{c'\in C, c'\neq N(\inputImage')} (\lipschitzConstant_{N(\inputImage')}+\lipschitzConstant_{c'})\distance{k}{\inputImage'-\inputImage''}
    \end{split}
\end{equation}
we can let 
\begin{equation}
\max_{c'\in C, c'\neq N(\inputImage')} (\lipschitzConstant_{N(\inputImage')}+\lipschitzConstant_{j})\distance{k}{\inputImage'-\inputImage''} \leq g(\inputImage',N(\inputImage')) 
\end{equation}
Thus, we define 
\begin{equation}
heuristic(\inputImage') =  \frac{g(\inputImage',N(\inputImage'))}{\max_{c'\in C, c'\neq N(\inputImage')} (\lipschitzConstant_{N(\inputImage')}+\lipschitzConstant_{c'})} 
\end{equation}
which is sufficient to ensure that $g(\inputImage'',N(\inputImage'))\geq 0$ for any $\inputImage''$. That is, the distance $heuristic(\inputImage')$ is a lower bound of reaching a misclassification.  
\end{proof}

\newcommand{\PlayerOne}{Player~$\playerOne$ }
\newcommand{\PlayerTwo}{Player~$\playerTwo$ }
\newcommand{\AtomicManipulation}{AtomicManipulation}
\newcommand{\DistanceEstimation}{DistanceEstimation}
\newcommand{\MaximumSafeRadius}{MaximumSafeRadius}

\begin{algorithm}[t]
\small
\caption{Admissible A* for DNN Verification}\label{AdmissibleA*}
\begin{algorithmic}[1]
\State \textbf{Input:} A game model $M(k,d,\inputImage,c)$, a termination condition $tc$
\State \textbf{Output:} $val(M(k,d,\inputImage,c), \maximumsaferadius(k,d,\inputImage,c))$
\Procedure{AdmissibleA*($M(k,d,\inputImage,c), tc$)}{}
\State $root \gets s_0$
\State \textbf{While}$(\neg tc)$: 
\State \quad \quad $features \gets$ \PlayerOne $(root,\text{feature extraction} =\mathsf{grey/black})$
\State \quad \quad \textbf{for} $feature$ in $features$: 
\State \quad \quad \quad \quad $dimensions \gets$ \PlayerTwo$(feature)$
\State \quad \quad \quad \quad $newnodes \gets \AtomicManipulation(dimensions)$ 
\State \quad \quad \quad \quad \textbf{for} $node$ in $newnodes$: 
\State \quad \quad \quad \quad \quad \quad $distances \gets \DistanceEstimation(node) $
\State \quad \quad $root \gets \MaximumSafeRadius(distances) $
\State \textbf{return} $\distance{k}{\inputImage(root)-\inputImage(s_0)}$
\EndProcedure
\end{algorithmic}
\end{algorithm}

The Admissible A* algorithm is presented in Algorithm~\ref{AdmissibleA*}. In the following, we explain the main components of the algorithm.
For each $root$ node (initialised as the original input), \PlayerOne chooses between mutually exclusive $features$ partitioned based on either the $\mathsf{grey}$-box or $\mathsf{black}$-box approach. Subsequently, in each $feature$, \PlayerTwo chooses among all the $dimensions$ within each feature (Line~4-8).
%
On each of the $dimensions$, an $\AtomicManipulation$ is constructed and applied.
We add $+\tau$ and $-\tau$ to each dimension, and make sure that it does not exceed the upper and lower bounds of the input dimension, e.g., 1 and 0 if the input is pre-processed (normalised). If exceeded, the bound value is used instead. This procedure essentially places adversarial perturbations on the image, and all manipulated images become the $newnodes$ (Line~9).
%
For each $node$ in the $newnodes$, the $\DistanceEstimation$ function in Equation~(\ref{equ:dist_estimation}) is used to compute a value, which is then added into the set $distances$. The set $distances$ maintains the estimated values for all leaf nodes (Line~10-11).  
%
Among all the leaf nodes whose values are maintained in $distances$, we select the one with the minimum $\MaximumSafeRadius$ as the new $root$ (Line~12). 

As for the termination condition $\neg tc$, the algorithm gradually unfolds the game tree with increasing tree depth $td=1,2,...$. Because all nodes on the same level of the tree have the same distance to the original input $\inputImage$,  every tree depth $td > 0$ is associated with a distance $d(td)$, 
such that 
$d(td)$ is the distance of the nodes at level $td$. 
For a given tree depth $td$, we have a termination condition $tc(td)$ requiring that either
\begin{itemize}
    \setlength\itemsep{0em}
    \item  all the tree nodes up to depth $td$ have been explored, or
    \item  the current $root$ is an adversarial example.
\end{itemize}
For the latter,
$\distance{k}{\inputImage(root)-\inputImage(s_0)}$ is returned and the algorithm converges. 
%
For the former, 
we update $d(td)$ as the current lower bound of the game value $val(M(k,d,\inputImage,c), \maximumsaferadius(k,d,\inputImage,c))$. 
\response{Note that 
the termination condition guarantees the closest adversarial example that corresponds to $\finitemaximumsaferadius$, which is within distance $\frac{1}{2}d(k, \tau)$ from the actual closest adversarial example corresponding to $\maximumsaferadius$.}

\newcommand{\Alpha}{\mathtt{alpha} }
\newcommand{\Beta}{\mathtt{beta} }

\subsection{Lower Bounds: Alpha-Beta Pruning in a Competitive Game}\label{sec:alphabetadescription}

Alpha-Beta Pruning is an adversarial search algorithm, applied commonly in two-player games, to minimise the possible cost in a maximum cost scenario. In this paper, we apply Alpha-Beta Pruning to compute the lower bounds of Player~$\playerOne$ reward in a \emph{competitive} game, i.e., $\featurerobustness_{\setoffeatures}(k,d,\inputImage,c)$.

\begin{lemma}
For any game state $s \in S \cup (S \times \setoffeatures(\inputImage))$, we let $\strategy_\playerOne(s) \in S \times \setoffeatures(\inputImage)$ be the next state of $s \in S$ after Player~$\playerOne$ taking an action $\strategy_\playerOne$, and $\strategy_\playerTwo(s) \in S$ be the next state of $s \in S \times \setoffeatures(\inputImage)$ after Player~$\playerTwo$ taking an action $\strategy_\playerTwo$. If using $\mathtt{alpha}(s)$ (initialised as $-\infty$) to denote Player~$\playerOne$ current maximum reward on state $s$ and $\mathtt{beta}(s)$ (initialised as $+\infty$) to denote Player~$\playerTwo$ current minimum reward on state $s$, and let
\begin{align}
    \mathtt{alpha}(s) & = \max_{\strategy_\playerOne} \mathtt{beta}(\strategy_\playerOne(s)) & \text{if } & s \in S\\
    \mathtt{beta}(s) & = \min_{\strategy_\playerTwo} \mathtt{alpha}(\strategy_\playerTwo(s)) & \text{if } & s \in S \times \setoffeatures(\inputImage)
\end{align}
then $\mathtt{alpha}(s_0)$ is a lower bound of the value $val(M(k,d,\inputImage,c), \featurerobustness_{\setoffeatures}(k,d,\inputImage,c))$.
\end{lemma}

Note that, for a game state $s$, whenever $\mathtt{alpha}(s) \geq \mathtt{beta}(s')$ for some $s' = \strategy_\playerOne(s)$ is satisfied, Player~$\playerOne$ does not need to consider the remaining strategies of Player~$\playerTwo$ on state $s'$, as such will not affect the final result. This is the pruning of the game tree. 
The Alpha-Beta Pruning algorithm is presented in Algorithm~\ref{Alpha-Beta}. 
Many components of the algorithm are similar to those of Admissible A*, except that each node maintains two values: $\Alpha$ value and $\Beta$ value. 
For every node, its $\Alpha$ value is initialised as $-\infty$ and its $\Beta$ value is initialised as $+\infty$. For each $feature$, its $\Beta$ value is the minimum of all the $\Alpha$ values of the perturbed inputs whose manipulated dimensions are within this feature (Line~14); for $root$ in each recursion, the $\Alpha$ value is the maximum  of all the $\Beta$ values of the features (Line~15). Intuitively, $\Beta$ maintains the $\maximumsaferadius$ of each feature, while $\Alpha$ maintains the $\featurerobustness_\setoffeatures$ of an input.

\begin{algorithm}[t]
\small
\caption{Alpha-Beta Pruning for DNN Verification}\label{Alpha-Beta}
\begin{algorithmic}[1]
\State \textbf{Input:} A game model $M(k,d,\inputImage,c)$, a termination condition $tc$
\State \textbf{Output:} $val(M(k,d,\inputImage,c), \featurerobustness_{\setoffeatures}(k,d,\inputImage,c))$
\Procedure{AlphaBeta($M(k,d,\inputImage,c), tc$)}{}
\State $root \gets s_0$
\State $root.\Alpha \gets -\infty$
\State $features \gets$ \PlayerOne $(root,\text{feature extraction} =\mathsf{grey/black})$
\State \textbf{for} $feature$ in $features$: 
\State \quad \quad $feature.\Beta \gets +\infty$
\State \quad \quad $dimensions \gets$ \PlayerTwo$(feature)$
\State \quad \quad $newnodes \gets \AtomicManipulation(dimensions)$ 
\State \quad \quad \textbf{for} $node$ in $newnodes$: 
\State \quad \quad \quad \quad \textbf{if} $tc$: \textbf{return} $\distance{k}{\inputImage(node)-\inputImage(s_0)}$
\State \quad \quad \quad \quad \textbf{else}: $node.\Alpha \gets $\textsc{AlphaBeta}$(node, tc)$
\State \quad \quad $feature.\Beta \gets \min (newnodes.\Alpha)$
\State $root.\Alpha \gets \max (features.\Beta)$
\State \textbf{return} $root.\Alpha$
\EndProcedure
\end{algorithmic}
\end{algorithm}

\subsection{Anytime Convergence}

In this section, we show the convergence of our approach, i.e., that both bounds are monotonically improved with respect to the optimal values. 

\paragraph{Upper Bounds: $1/\epsilon$-Convergence and Practical Termination Condition ($tc$)} 
Because we are working with a finite game, MCTS is guaranteed to converge when the game tree is fully expanded, but the worst case convergence time may be prohibitive. 
In practice, we can work with $1/\epsilon$-convergence by letting the program terminate when the current best bound has not been improved 
for e.g., $\lceil 1/\epsilon \rceil$  iterations, where $\epsilon>0$ is a small real number. We can also impose time constraint $tc$, and ask the program to return once the elapsed time of the computation has exceeded $tc$. 

In the following, we show that the intermediate results from Algorithm~\ref{MCTS} can be the upper bounds of the optimal values, and the algorithm is continuously improving the upper bounds, until the optimal values are reached. 

\begin{lemma}\label{lemma:lsconvergence}
Let $\distance{k}{\inputImage'- e_{root}}$ be the returned result from Algorithm~\ref{MCTS}. For an $\finitemaximumsaferadius(k,d,\inputImage,c)$ game, we have that 
\begin{equation}\label{equ:lsconvergence}
\distance{k}{\inputImage'- e_{root}}\geq val(M(k,d,\inputImage,c), \maximumsaferadius(k,d,\inputImage,c)).
\end{equation}
Moreover, the discrepancy between $\distance{k}{\inputImage'- e_{root}}$ and $val(M(k,d,\inputImage,c), \maximumsaferadius(k,d,\inputImage,c))$ improves monotonically as the computation proceeds. 
\end{lemma}
\begin{proof}
Assume that we have a partial tree $\partialTree(k,d,\inputImage,c)$. We prove by induction on the structure of the tree. As the base case, for each leaf node $o$ we have that its best input $e_o$ is such that 
\begin{equation}\label{equ:lsproof1}
\distance{k}{\inputImage - e_o} \geq val(M(k,d,\inputImage,c), \maximumsaferadius(k,d,\inputImage,c))
\end{equation}
because a random simulation can always return a current best, which is an upper bound to the global optimal value. 
The equivalence holds when the simulation found an adversarial example with minimum distance.

Now, for every internal node $o$, by Equation~(\ref{equ:lsbestcase}) we have that 
\begin{equation}
\exists o_1\in\children_o: \distance{k}{\inputImage-e_o} \geq  \distance{k}{\inputImage - e_{o_1}} 
\end{equation}
which, together with Equation~(\ref{equ:lsproof1}) and induction hypothesis, implies that $\distance{k}{\inputImage - e_o} \geq val(M(k,d,\inputImage,c), \maximumsaferadius(k,d,\inputImage,c))$. Equation~(\ref{equ:lsconvergence}) holds since the root node is also an internal node. 

The monotonic improvement 
can be seen from Equation~(\ref{equ:lsbestcase}), namely that, when, and only when, the discrepancy for the leaf node is improved after a new round of random simulation, can the discrepancy for the root node be improved. Otherwise, it remains the same. 
\end{proof}

Similarly, we have the following lemma for the feature robustness game.
\begin{lemma}
Let $\distance{k}{\inputImage'-e_{root}}$ be the returned result from Algorithm~\ref{MCTS}. For an $\finitefeaturerobustness_{\setoffeatures}(k,d,\inputImage,c)$ game, we have that 
\begin{equation}
\distance{k}{\inputImage'-e_{root}}\geq val(M(k,d,\inputImage,c), \featurerobustness_{\setoffeatures}(k,d,\inputImage,c))
\end{equation}
\end{lemma}
\begin{proof}
The proof is similar to that of Lemma~\ref{lemma:lsconvergence}, except that, according to Equation~(\ref{equ:frbestcase}), for the nodes of Player~$\playerOne$ (including the root node) to reduce the discrepancy, i.e., $\distance{k}{\inputImage - e_o} - val(M(k,d,\inputImage,c), \featurerobustness_{\setoffeatures}(k,d,\inputImage,c))$, it requires that all its children nodes  reduce their discrepancy. 
\end{proof}

\paragraph{Lower Bounds: Gradual Expansion of the Game Tree}

The monotonicity of the lower bounds is achieved by gradually increasing the tree depth $td$. Because, in both algorithms, the termination conditions are the full exploration of the partial trees up to the depth $td$, it is straightforward that the results returned by the algorithms are either the lower bounds or the converged results.

\section{Experimental Results}\label{sec:results}

This section presents experimental results for the proposed game-based approach for safety verification of deep neural networks, focused on demonstrating convergence and comparison with state-of-the art techniques. 

\subsection{Feature-Based Partitioning}\label{exp:feature}

\begin{figure}[t]
	\centering
	\includegraphics[width=1\linewidth]{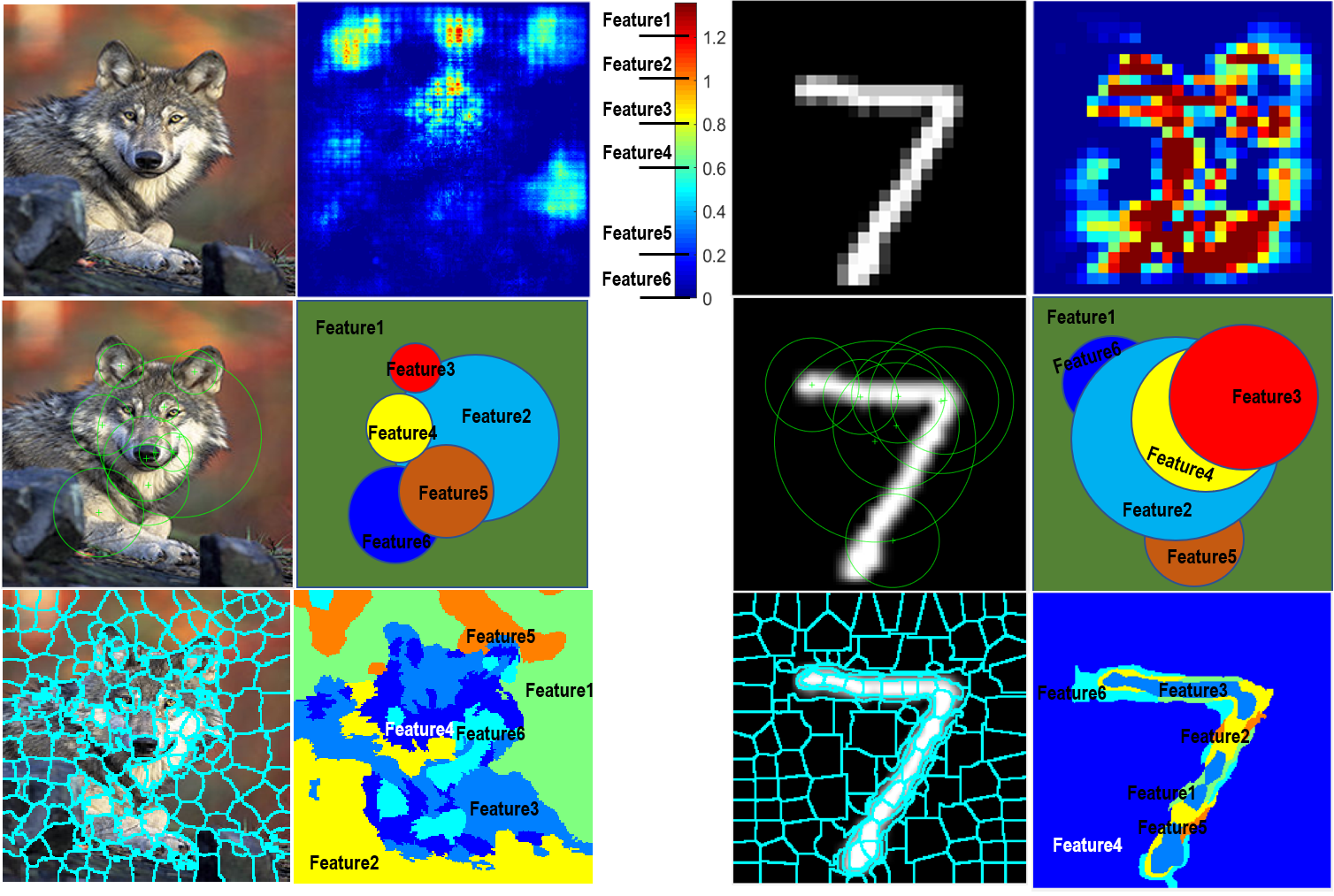}
	\caption{Examples of three feature extraction methods applied on the ImageNet (left) and MNIST (right) datasets, respectively. The top row: image segmentation via the saliency map generation method introduced in~\cite{RWSHKK2018}. Middle row: feature extraction using the SIFT approach~\cite{SIFT}. Bottom row: image partition using K-means clustering and superpixels~\cite{wei2015superpixels}.} 
	\label{fig:feature}
\end{figure}

Our game-based approach, 
where \PlayerOne determines features and \PlayerTwo selects pixels or dimensions within the selected feature, requires an appropriate feature partitioning method into disjoint sets of dimensions. 
In Figure~\ref{fig:feature} we illustrate three distinct feature extraction procedures on a colour image from the ImageNet dataset and a grey-scale  image from the MNIST dataset.
Though we work with image classifier networks, our approach is flexible and can be adapted to a range of feature partitioning methods.

The first technique for image segmentation is based on the saliency map generated from an image classifier such as a DNN. As shown Figure~\ref{fig:feature} (top row), the heat-map is produced by quantifying how sensitive each pixel is to the classification outcome of the DNN. By ranking these sensitivities, we separate the pixels into a few disjoint sets. 
The second feature extraction approach, shown in Figure~\ref{fig:feature} (middle row), is independent of any image classifier,
but instead focuses on abstracting the invariant properties directly from  the image. Here we show 
segmentation results from the SIFT method~\cite{SIFT}, which is invariant to image translation, scaling, rotation, and local geometric distortion. More details on how to adapt SIFT for safety verification on DNNs can be found in \cite{WHK2017}. The third feature extraction method is based on superpixel representation, a dimensionality reduction technique widely applied in various computer vision applications. Figure~\ref{fig:feature} (bottom row) demonstrates an example of how to generate superpixels (i.e., the pixel clusters marked by the green grids) using colour features and K-means clustering~\cite{wei2015superpixels}. 


\response{
\subsection{Lipschitz Constant Estimation}\label{lip-discussion}

Our approach assumes knowledge of a (not necessarily tight) Lipschitz constant.  Several techniques can be used to estimate such a constant, including 
FastLin/FastLip \cite{Weng2018TowardsFC}, Crown \cite{Zhang2018EfficientNN} and DeepGO \cite{RHK2018}. For more information see the Related Work section.

The size of the Lipschitz constant is inversely proportional to the number of grid points and error bound, and therefore affects computational performance. We remark that, due to the high non-linearity and high-dimensionality of  modern DNNs, it is non-trivial to conduct verification even if the Lipschitz constant is known. 
}

\subsection{Convergence Analysis of the Upper and Lower Bounds}

\begin{figure}[t]
	\centering
	\includegraphics[width=1\linewidth]{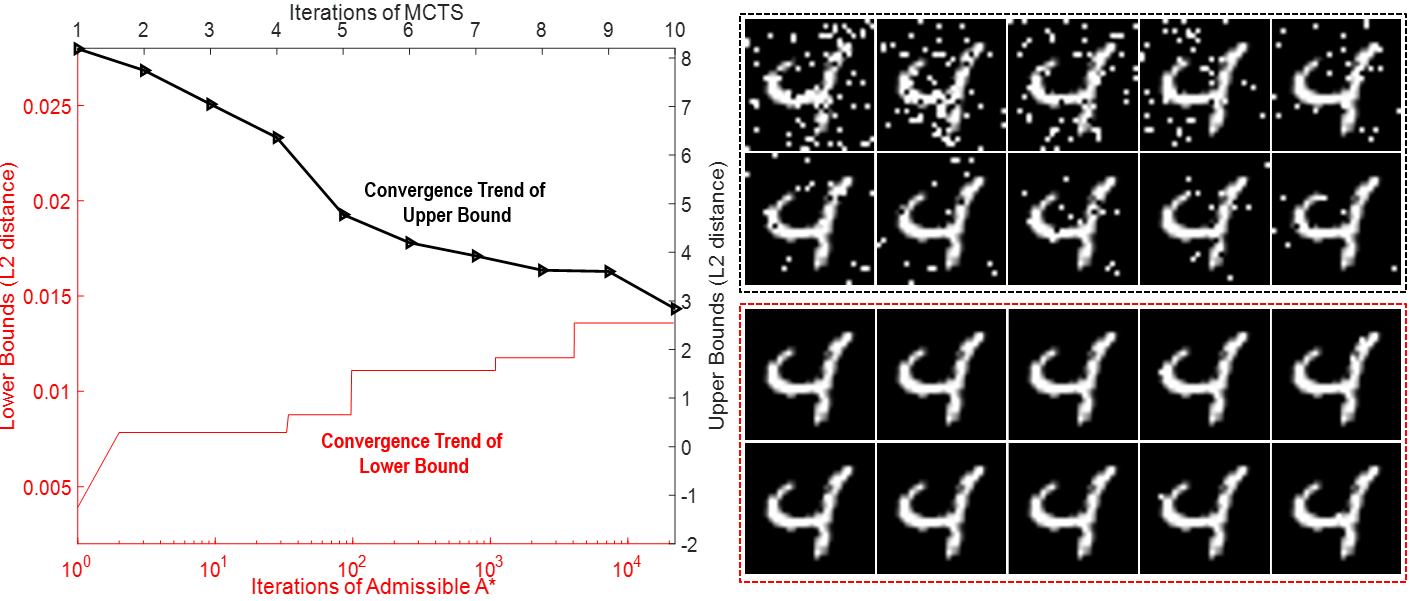}
	\caption{Convergence of \emph{maximum safe radius} in a \emph{cooperative} game with $\mathsf{grey}$-box feature extraction of a MNIST image originally classified as ``4''. Left: The convergence trends of the upper bound from MCTS and the lower bound from Admissible A* for the $\maximumsaferadius$ problem. Right: The generated \emph{adversarial} images while searching for the upper bound via MCTS, and lower boundary \emph{safe} images while searching for the lower bound via Admissible A*.} 
	\label{fig:converge_coop}
\end{figure}

\begin{figure}[t]
	\centering
	\includegraphics[width=1\linewidth]{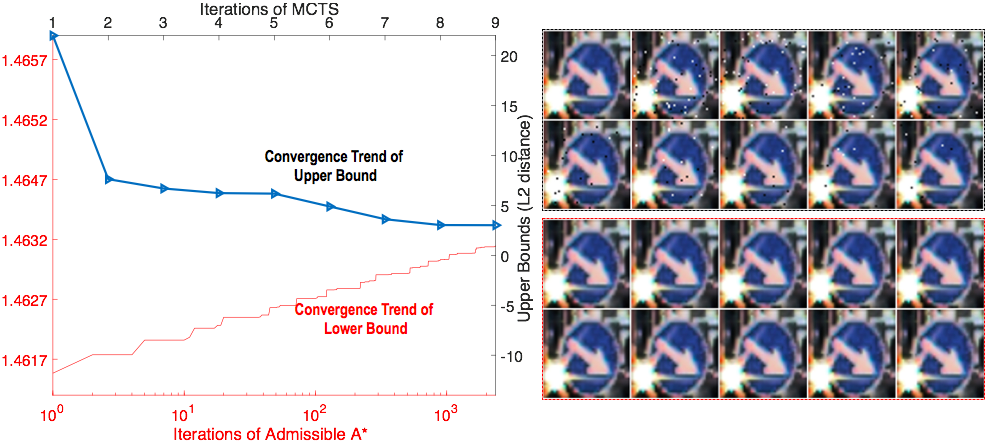}
	\caption{Convergence of \emph{maximum safe radius} in a \emph{cooperative} game with $\mathsf{grey}$-box feature extraction of a GTSRB image originally classified as ``keep right''. Left: The convergence trends of the upper bound from MCTS and the lower bound from Admissible A* for the $\maximumsaferadius$ problem. Right: The generated \emph{adversarial} images while searching for the upper bound via MCTS, and lower boundary \emph{safe} images while searching for the lower bound via Admissible A*.} 
	\label{fig:converge_coop_gtsrb}
\end{figure}

We demonstrate convergence of the bound computation for the maximum safe radius and feature robustness problems, evaluated on standard benchmark datasets \response{MNIST, CIFAR10, and GTSRB. The architectures of the corresponding trained neural networks as well as their accuracy rates can be found in \ref{app:architectures}.}

\paragraph{Convergence of $\maximumsaferadius$ in a Cooperative Game}

First, we illustrate convergence of $\maximumsaferadius$ in a \emph{cooperative} game on the MNIST and GTSRB datasets. For the MNIST image (index 67) in Figure~\ref{fig:converge_coop}, the black line denotes the descending trend of the upper bound $\upperbound_\maximumsaferadius$, whereas the red line indicates the ascending trend of the lower bound $\lowerbound_\maximumsaferadius$. Intuitively, after a few iterations, the upper bound (i.e., minimum distance to an adversarial example) is 2.84 wrt the $L_2$ metric, and the absolute safety (i.e., lower bound) is within radius 0.012 from the original image. The right-hand side of Figure~\ref{fig:converge_coop} includes images produced by intermediate iterations, with \emph{adversarial} images generated by MCTS shown in the two top rows, and \emph{safe} images computed by Admissible A* in the bottom rows.
\response{
Similarly, Figure~\ref{fig:converge_coop_gtsrb} displays the converging upper and lower bounds of $\maximumsaferadius$ in a \emph{cooperative} game on a GTSRB image (index 19).
}

As for the computation time, each MCTS iteration updates the upper bound $\upperbound_\maximumsaferadius$ and typically takes minutes; each Admissible A* iteration further expands the game tree and updates the lower bound $\lowerbound_\maximumsaferadius$ whenever applicable. The running times for the iterations of the Admissible A* vary: initially it takes minutes but this can increase to hours when the tree is 
larger.

\begin{figure*}[t]
\centering
    \begin{subfigure}{0.5\linewidth}
        \centering
	    \includegraphics[width=1\linewidth]{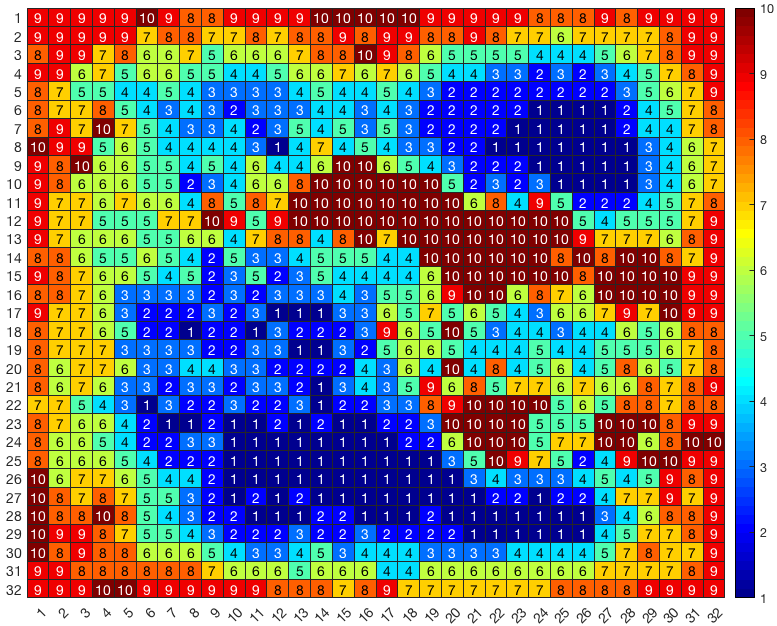}
	    \caption{CIFAR10 ``ship'' image.}
	    \label{fig:converge_comp_feature_cifar10}
 	\end{subfigure}%
 	\begin{subfigure}{0.5\linewidth}
 	    \centering
        \includegraphics[width=1\linewidth]{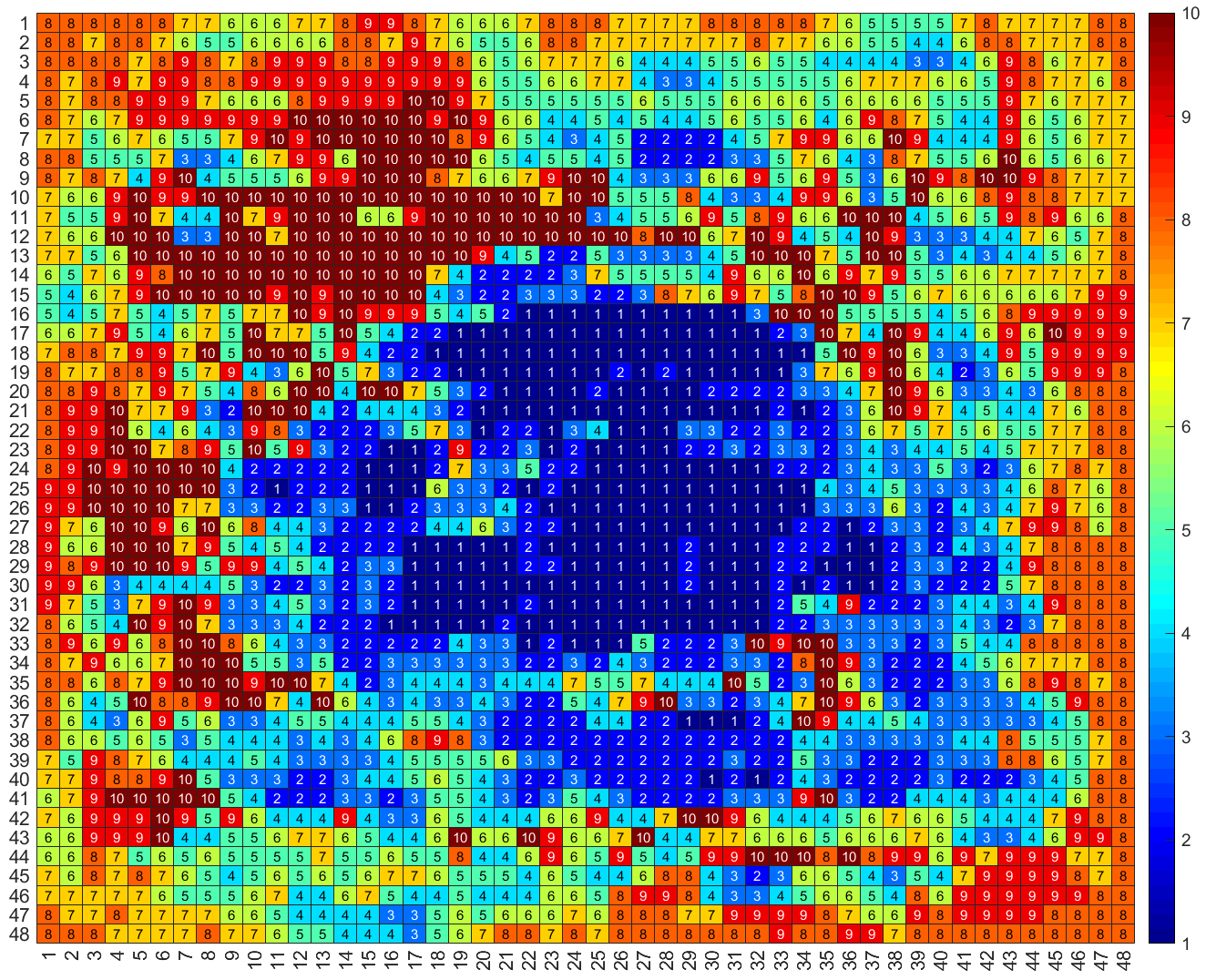}
        \caption{GTSRB ``speed limit 70'' image.}
    	\label{fig:converge_comp_feature_gtsrb}
 	 \end{subfigure}
     \caption{Illustration of the 10 features using the $\mathsf{grey}$-box feature extraction procedure: cells with the same colour indicate the same feature, and the number in each cell represents the $\mathsf{featureID}$. That is, $\mathsf{Feature1}$ in deep blue has the most salient impact, whereas $\mathsf{Feature10}$ in deep red is the least influential. (a) Features of the CIFAR10 ``ship'' image ($32 \times 32$) in Figure~\ref{fig:converge_comp_cifar10}. (b) Features of the GTSRB ``speed limit 70'' image ($48 \times 48$) in Figure~\ref{fig:converge_comp_gtsrb}.} 
     \label{fig:fig:converge_comp_feature}
 \end{figure*}

\begin{figure}[t]
	\centering
	\includegraphics[width=1\linewidth]{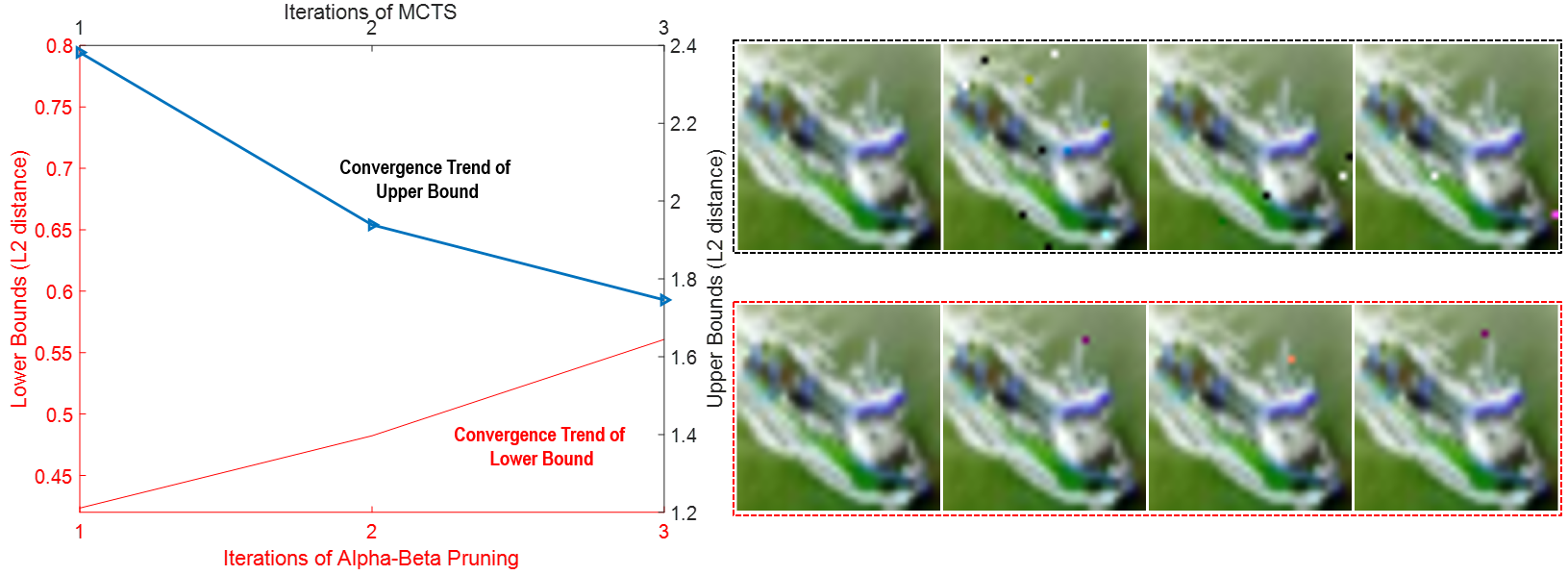}
	\caption{Convergence of \emph{feature robustness} in a \emph{competitive} game with $\mathsf{grey}$-box feature extraction of a CIFAR10 image originally classified as ``ship''. Left: The convergence trends of the upper bound from MCTS and the lower bound from Alpha-Beta Pruning for the $\featurerobustness_\setoffeatures$ problem. Right: The generated adversarial images while computing the upper bounds via MCTS, and lower bound images while computing  the lower bounds via Alpha-Beta Pruning.} 
	\label{fig:converge_comp_cifar10}
    \vspace{1em}
	\centering
	\includegraphics[width=1\linewidth]{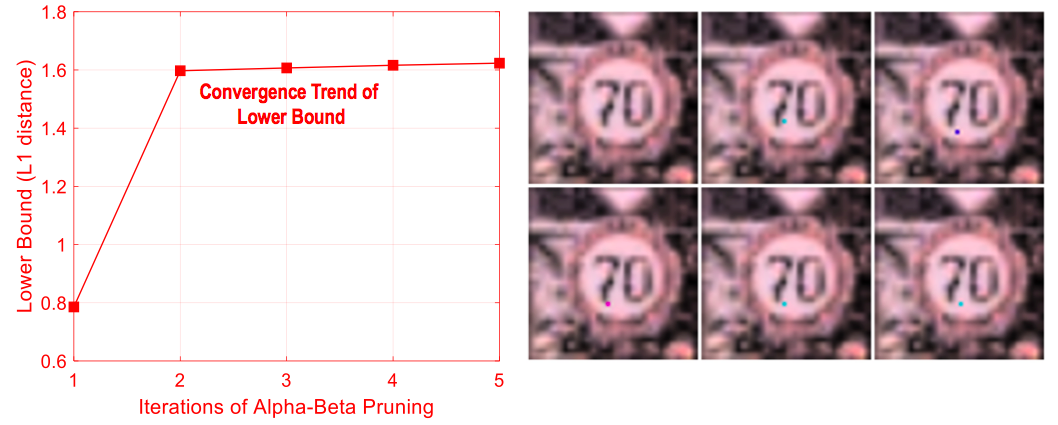}
	\caption{Convergence of \emph{feature robustness} in a \emph{competitive} game with the $\mathsf{grey}$-box feature extraction of a GTSRB image originally classified as ``speed limit 70''. Left: The convergence trends of the lower bound from Alpha-Beta Pruning for the $\featurerobustness_\setoffeatures$ problem. Right: The generated lower bound images while computing the lower bounds via Alpha-Beta Pruning.} 
	\label{fig:converge_comp_gtsrb}
\end{figure}

\paragraph{Convergence of $\featurerobustness_\setoffeatures$ in a Competitive Game}

Next we demonstrate the convergence of $\featurerobustness_\setoffeatures$ in a \emph{competitive} game on the CIFAR10 and GTSRB datasets. Each iteration of MCTS or Alpha-Beta Pruning updates their respective bound with respect to a certain feature. Note that, in each MCTS iteration, upper bounds $\upperbound_\maximumsaferadius$ of all the features are improved and therefore the maximum among them, i.e., $\upperbound_{\featurerobustness_\setoffeatures}$ of the image, is updated, whereas Alpha-Beta Pruning calculates $\lowerbound_\maximumsaferadius$ of a feature in each iteration, and then compares and updates $\lowerbound_{\featurerobustness_\setoffeatures}$ with the computation progressing until all the features are processed. 

For the CIFAR10 image in Figure~\ref{fig:converge_comp_cifar10}, the green line denotes the upper bound $\upperbound_{\featurerobustness_\setoffeatures}$ and the red line denotes the lower bound $\lowerbound_{\featurerobustness_\setoffeatures}$. The ``ship'' image is partitioned into $10$ features (see Figure~\ref{fig:converge_comp_feature_cifar10}) utilising the $\mathsf{grey}$-box extraction method. We observe that this saliency-guided image segmentation procedure captures the features well, as in Figure~\ref{fig:converge_comp_feature_cifar10} the most influential features (in blue) resemble the silhouette of the ``ship''. After 3 iterations, the algorithm indicates that, at $L_2$ distance of more than 1.75, all features are fragile, and if the $L_2$ distance is 0.48  there exists at least one robust feature. The right-hand side of Figure~\ref{fig:converge_comp_cifar10} shows several  intermediate images produced, along with the converging $\upperbound_{\featurerobustness_\setoffeatures}$ and $\lowerbound_{\featurerobustness_\setoffeatures}$. The top row exhibits the original image as well as the manipulated images with decreasing $\upperbound_\featurerobustness$. For instance, after the 1st iteration, MCTS finds an adversary perturbed in $\mathsf{Feature4}$ with $L_2$ distance 2.38, which means by far the most robust feature of this ``ship'' image is $\mathsf{Feature4}$. ($\mathsf{FeatureID}$ is retrieved from the number in each cell of the image segmentation in Figure~\ref{fig:converge_comp_feature_cifar10}.) When the computation proceeds, the 2nd iteration updates $\upperbound_{\featurerobustness_\setoffeatures}$ from 2.38 to 1.94, and explores the current most robust $\mathsf{Feature8}$, which is again replaced by $\mathsf{Feature9}$ after the 3rd iteration with lower distance 1.75. The bottom row displays the original image together with perturbations in each feature while $\lowerbound_{\featurerobustness_\setoffeatures}$ is increasing. It can be seen that $\mathsf{Feature1}$, $\mathsf{Feature2}$, and $\mathsf{Feature3}$ need only one dimension change to cause image misclassification, and the lower bound $\mathsf{Feature4}$ increases from 0.42 to 0.56 after three iterations.

For the \emph{feature robustness} ($\featurerobustness_\setoffeatures$) problem, i.e., when Player~$\playerOne$ and Player~$\playerTwo$ are competing against each other, apart from the previous CIFAR10 case where Player~$\playerTwo$ wins the game by generating an adversarial example with atomic manipulations in each feature, there is a chance that Player~$\playerOne$ wins, i.e., at least one robust feature exists. Figure~\ref{fig:converge_comp_gtsrb} illustrates this scenario on the GTSRB dataset. Here Player~$\playerOne$ defeats Player~$\playerTwo$ through finding at least one robust feature by MCTS, and thus the convergence trend of the upper bound $\upperbound_{\featurerobustness_\setoffeatures}$ is not shown. As for the lower bound $\lowerbound_{\featurerobustness_\setoffeatures}$, Alpha-Beta Pruning enables Player~$\playerTwo$ to manipulate a single pixel in $\mathsf{Feature1}$ - $\mathsf{Feature5}$ (see Figure~\ref{fig:converge_comp_feature_gtsrb}) so that adversarial examples are found. For instance, with $L_1$ distance above 0.79, $\mathsf{Feature1}$ turns out to be fragile.

Here, each iteration of MCTS or Alpha-Beta Pruning is dependent on the size of feature partitions -- for smaller partitions it takes seconds to minutes, whilst for larger partitions it can take hours. The running times are also dependent on the norm ball radius $d$. If the radius $d$ is small, the computation can always terminate in minutes.

\response{

\paragraph{Scalability wrt Number of Input Dimensions}

\begin{figure}[t]
    \centering
    \includegraphics[width=0.6\linewidth]{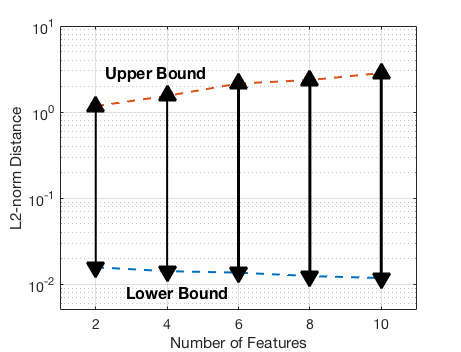}
    \caption{\response{Analysis of convergence of upper and lower bounds of the \emph{maximum safe radius} as the number of dimensions increases for the MNIST image in Figure~\ref{fig:converge_coop}, based on the $L_2$ norm. The increase in the number of features (2 to 10) corresponds to an increase in the number the input dimensions.}}
    \label{fig:scalability}
\end{figure}

We now investigate how the increase in the number of dimensions affects the convergence of the lower and upper bounds.
From the complexity analysis of the problems in Section~\ref{sec:complexity}, we know that the theoretical complexity is in PTIME with respect to the size of the game model, which is exponential with respect to the number of input dimensions. 

We utilise the example in Figure~\ref{fig:converge_coop}, where the convergence of the upper bound $\upperbound_\maximumsaferadius$ and the lower bound $\lowerbound_\maximumsaferadius$ in a cooperative game is exhibited on \emph{all} the dimensions (pixels) of the MNIST image (index 67). We partition the image into $10$ disjoint features using the $\mathsf{grey}$-box extraction method, and gradually manipulate features, starting from those with fewer dimensions, 
to observe how the corresponding bound values $\upperbound_\maximumsaferadius$, $\lowerbound_\maximumsaferadius$ are affected if we fix a time budget.
To ensure fair comparison, we run the same number of expansions of the game tree, i.e., $10$ iterations of MCTS, and $1000$ iterations of Admissible A*, and plot the bound values $\upperbound_\maximumsaferadius$, $\lowerbound_\maximumsaferadius$ thus obtained. 
Figure~\ref{fig:scalability} shows the widening upper and lower bounds based on the $L_2$ norm with respect to $2$ 
to $10$ features of the image. 
It is straightforward to see that the conclusion also holds for the \emph{feature robustness} problem.

}

\subsection{Comparison with Existing Approaches for Generating Adversarial Examples}
\label{subsec:L0Comparison}

\begin{figure}[t]
	\centering
	\includegraphics[width=0.8\linewidth]{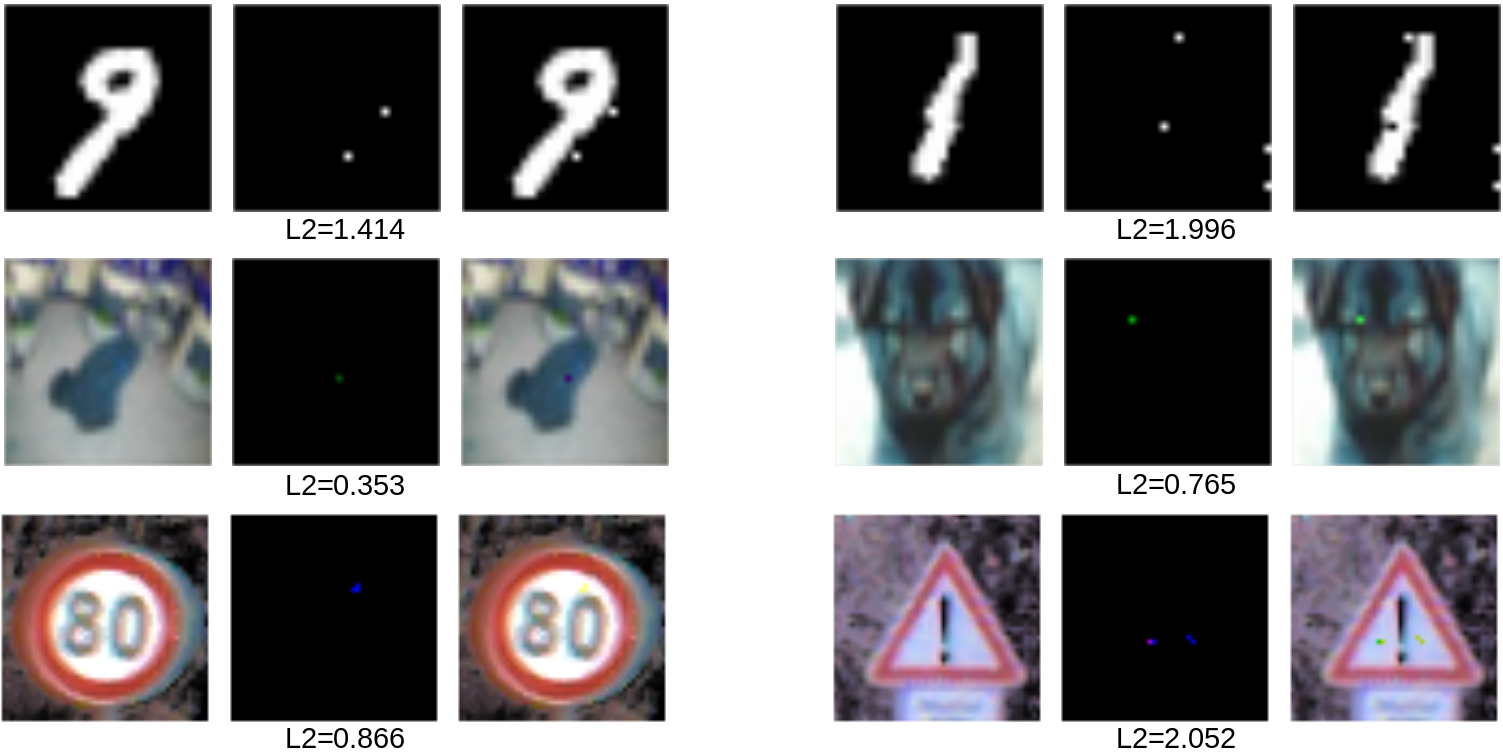}
	\caption{Examples of adversarial MNIST, CIFAR10 and GTSRB images with slight perturbations based on the $L_2$-norm. Top: ``9'' misclassified into ``8''; ``1'' misclassified into ``3''. Middle: ``frog'' misclassified into ``dog''; ``dog'' misclassified into ``cat''. Bottom: ``speed limit 80'' misclassified into ``speed limit 60''; ``danger'' misclassified into ``pedestrian crossing''.} 
	\label{fig:adv_mnist+cifar10+gtsrb}
\end{figure}

When the game is cooperative, i.e., for the \emph{maximum safe radius} problem, we can have adversarial examples as by-products. In this regard, both MCTS and A* algorithm can be applied to generate adversarial examples. Note that, for the latter, we can take Inadmissible A* (i.e., the heuristic function can be inadmissible), as the goal is not to ensure the lower bound but to find adversarial examples. By proportionally enlarging the heuristic distance $heuristic(\inputImage')$ with a constant, we ask the algorithm to explore those tree nodes where an adversarial example is more likely to be found. Figure~\ref{fig:adv_mnist+cifar10+gtsrb} displays some adversarial MNIST, CIFAR10 and GTSRB images generated by $\mathsf{DeepGame}$ after  manipulating a few pixels. More examples can be found in Figures~\ref{fig:adversary+mnist}, \ref{fig:adversary+cifar10}, and \ref{fig:adversary+gtsrb} in the Appendix.

We compare our tool $\DeepGame$ with several state-of-the-art approaches to search for adversarial examples: CW~\cite{CW-Attacks}, L0-TRE~\cite{RWSHKK2018}, DLV~\cite{DLV}, SafeCV~\cite{WHK2017}, and JSMA~\cite{JSMA}. More specifically, we train neural networks on two benchmark datasets, MNIST and CIFAR10, and calculate the distance between the adversarial image and the original image based on the $L_0$-norm. The original images, preprocessed to be within the bound $[0, 1]$, are the first 1000 images of each testing set. Apart from a ten-minute time constraint, we evaluate on correctly classified images and their corresponding adversarial examples. This is because some tools regard misclassified images as adversarial examples and record zero-value distance while other tools do not, which would result in unfair comparison. The hardware environment is a Linux server with NVIDIA GeForce GTX TITAN Black GPUs, and the operating system is Ubuntu 14.04.3 LTS.

Table~\ref{tbl:comparison} demonstrates the statistics.   
Figure~\ref{fig:adversary+mnist} and Figure~\ref{fig:adversary+cifar10} in the Appendix include adversarial examples found by these tools. Model architectures, descriptions of the datasets and baseline methods, together with the parameter settings for these tools, can be found in \ref{app:comparison}. 

\begin{table}[t]
    \caption{Comparison between our tool $\DeepGame$ and several other tools on search for adversarial examples performed on the MNIST and CIFAR10 datasets, based on the L0-norm. Here $\DeepGame$ deploys the $\mathsf{grey}$-box feature extraction method, and Inadmissible A* algorithm. We set a ten-minute time constraint and evaluate on correctly classified images and the produced adversarial examples.}
    \label{tbl:comparison}
    \center
    \def\arraystretch{1.2}
    \begin{tabular}{c|c c|c c|c c|c c}
    \toprule
        \multirow{3}{*}{$L_0$} & \multicolumn{4}{c|}{MNIST} & \multicolumn{4}{c}{CIFAR10\footnotemark} \\ \cline{2-9}
        & \multicolumn{2}{c|}{Distance} & \multicolumn{2}{c|}{Time(s)} & \multicolumn{2}{c|}{Distance} & \multicolumn{2}{c}{Time(s)} \\ \cline{2-9}
        & mean & std & mean & std & mean & std & mean & std \\ \hline
        $\mathsf{DeepGame}$ & \textbf{6.11} & \textbf{2.48} & \textbf{4.06} & \textbf{1.62} & \textbf{2.86} & \textbf{1.97} & \textbf{5.12} & \textbf{3.62} \\
        CW & 7.07 & 4.91 & 17.06 & 1.80 & 3.52 & 2.67 & 15.61 & 5.84 \\
        L0-TRE & 10.85 & 6.15 & 0.17 & 0.06 & 2.62 & 2.55 & 0.25 & 0.05 \\
        DLV & 13.02 & 5.34 & 180.79 & 64.01 & 3.52 & 2.23 & 157.72 & 21.09 \\
        SafeCV & 27.96 & 17.77 & 12.37 & 7.71 & 9.19 & 9.42 & 26.31 & 78.38 \\
        JSMA & 33.86 & 22.07 & 3.16 & 2.62 & 19.61 & 20.94 & 0.79 & 1.15 \\
    \bottomrule
\end{tabular}
\end{table}
\footnotetext{Whilst $\DeepGame$ works on channel-level dimension of an image, in order to align with some tools that attack at pixel level the statistics are all based on the number of different pixels.}

\subsection{Evaluating Safety-Critical Networks}

\begin{figure}[t]
    \centering
    \includegraphics[width=\textwidth]{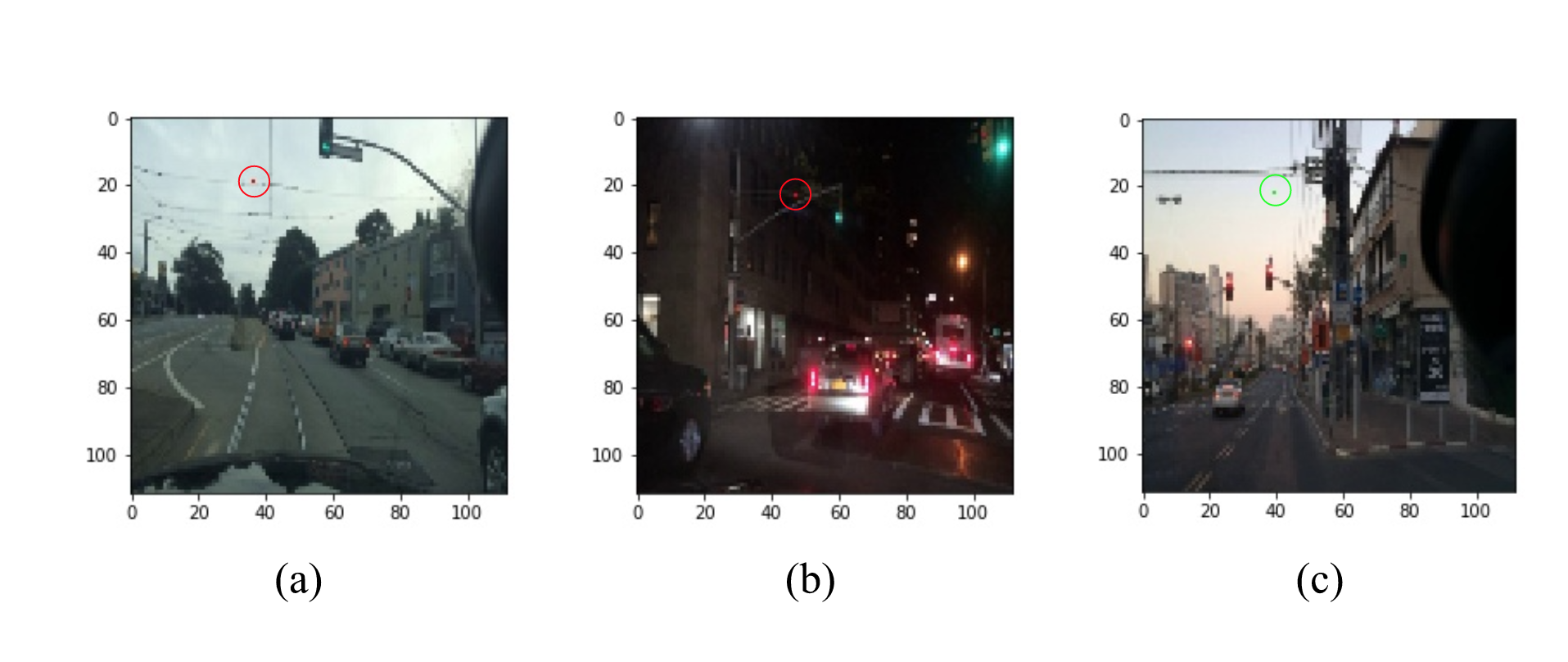}
    \caption{Adversarial examples generated on Nexar data demonstrate a lack of robustness. (a) Green light classified as red with confidence 56\% after one pixel change. (b) Green light classified as red with confidence 76\% after one pixel change. (c) Red light classified as green with 90\% confidence after one pixel change.}
    \label{fig:NexarFig1}
    
\end{figure}

We explore the possibility of applying our game-based approach  to support real-time decision making and testing,  
for which the algorithm needs to be highly efficient, requiring only seconds to execute a task. 

We apply our method to a network used for classifying traffic light images collected from dashboard cameras. The Nexar traffic light challenge~\cite{NexarData} made over eighteen thousand dashboard camera images publicly available. Each image is labelled either green, if the traffic light appearing in the image is green, or red, if the traffic light appearing in the image is red, or null if there is no traffic light appearing in the image. We test the winner of the challenge which scored an accuracy above 90\%~\cite{NexarEntry}. 
Despite each input being 37632-dimensional ($112 \times 112 \times 3$), our algorithm reports that the manipulation of an average of 4.85 dimensions changes the network classification. 
We illustrate the results of our analysis of the network in Figure~\ref{fig:NexarFig1}. Although the images are easy for humans to classify, only one pixel change causes the network to make potentially disastrous decisions, particularly for the case of red light misclassified as green. 
To explore this particular situation in greater depth, we use a targeted safety MCTS procedure on the same 1000 images, aiming to manipulate images into green. We do not consider images which are already classified as green. Of the remaining 500 images, our algorithm is able to change all image classifications to green with worryingly low distances,  namely an average $L_0$ of 3.23.


\section{Related Work}\label{sec:related}
In this section we review works related to safety and robustness verification for neural networks, Lipschitz constant estimation and feature extraction.


\subsection{White-box Heuristic Approaches}
In \cite{propertiesOfNeuralNetworks}, Szegedy et. al. find a targeted adversarial example by running the L-BFGS algorithm, which minimises the  $L_2$ distance between the images while maintaining the misclassification.
Fast Gradient Sign Method (FGSM) \cite{FGSM}, a refinement of L-BFGS, 
 takes as inputs the parameters $\theta$ of the model, the input $\inputImage$ to the model, and the target label $y$, and computes a linearized version of the cost function with respect to $\theta$ to obtain a manipulation direction. After the manipulation direction is fixed, a small constant value $\tau$ is taken as the magnitude of the manipulation. 
Carlini and Wagner \cite{CW-Attacks} adapt the optimisation problem proposed in \cite{propertiesOfNeuralNetworks} to obtain a set of optimisation problems for $L_0$, $L_2$, and $L_\infty$ attacks. They claim  better performance than FGSM and Jacobian-based Saliency Map Attack (JSMA) with their $L_2$ attack, in which for every pixel $x_i$ a new real-valued variable $w_i$ is introduced and then the optimisation is conducted by letting $x_i$ move along the gradient direction of $\tanh(w_i)$. 
Instead of optimisation, JSMA \cite{JSMA} uses 
a loss function to create a ``saliency map" of the image, 
which indicates the importance of each 
pixel on the network's decision. 
A greedy algorithm is used to gradually modify the most important pixels.
In \cite{FGSM-universal}, an iterative application of an optimisation approach (such as \cite{propertiesOfNeuralNetworks})  is conducted on a set of images one by one to get an accumulated manipulation, which is expected to make a number of inputs misclassified.  
\cite{DBLP:journals/corr/abs-1708-06939} replaces the softmax layer in a deep network with a multiclass SVM and then finds adversarial examples by performing a gradient computation. 

\subsection{White-box Verification Approaches}
Compared to heuristic search for adversarial examples, verification approaches aim to provide guarantees on the safety of DNNs. 
An early verification approach~\cite{PT2010}
 encodes the entire network as a set of constraints. The constraints can then be solved with a SAT solver. 
\cite{KBDJK2017} improves on \cite{PT2010} by handling the ReLU activation functions. The Simplex method for linear programming is extended  
to work with the piecewise linear ReLU functions that cannot be expressed using linear programming. 
The approach can scale up to networks with 300 ReLU nodes. 
In recent work~\cite{GKPB2017} the input vector space is partitioned using clustering and then the method of \cite{KBDJK2017} is used to check the individual partitions. 
DLV~\cite{DLV} uses multi-path search and layer-by-layer refinement to exhaustively explore a finite region of  the vector spaces associated with the input layer or the hidden layers, and scales to work with state-of-the-art networks such as VGG16.  
\cite{RHK2018} shows that most known layers of DNNs are Lipschitz continuous and presents a verification approach based on global optimisation. 
\response{
DiffAI~\cite{pmlr-v80-mirman18b} is a method for training robust neural networks based on abstract interpretation, but is unable to calculate the maximum safe radius $\maximumsaferadius$.}

\response{
\subsection{Lipschitz Continuity}

The idea of using Lipschitz continuity to provide guarantees on output behaviour of neural networks has been known for some time. Early work \cite{old-lip-1,old-lip-2}    
focused on small neural networks (few neurons and layers) that only contain differentiable activation functions such as the sigmoid. These works are mainly concerned with the computation of a Lipschitz constant based on strong assumptions that the network has a number of non-zero derivatives and a finite order Taylor series expansion can be found at each vertex. 
In contrast, our approach assumes knowledge of a (not necessarily tight) Lipschtz constant and focuses
on developing verification algorithms for realistically-sized modern networks that use ReLU activation functions, which are non-differentiable.  
}

\response{
\subsection{Lipschitz Constant Estimation}

There has been a resurgence of interest in Lipschitz constant estimation for neural networks.
The approaches of 
FastLin/FastLip \cite{Weng2018TowardsFC} and Crown \cite{Zhang2018EfficientNN} 
aim to estimate 
the Lipschitz constant by considering the analytical form of the DNN layers. They are able to compute the bounds, but, in contrast to our approach, which gradually improves the bounds, are not able to improve them. Moreover, their algorithms 
require access to complete information (e.g., architecture, parameters, etc) about the DNN, while our approach is mainly ``black-box'' with a (not necessarily tight) Lipschitz constant. We remark that a loose Lipschitz constant can be computed quite easily, noting that a tighter constant can improve computational performance. 

Although estimation of the Lipschitz constant has not been the focus of this paper, knowledge of the Lipschitz constant is important in safety verification of DNNs (i.e., estimation of $\maximumsaferadius$). 
The recent tool DeepGO \cite{RHK2018} develops a dynamic Lipschitz constant estimation method for DNNs by taking advantage of advances in Lipschitzian optimisation, through which we can construct both lower and upper bounds for $\maximumsaferadius$ with the guarantee of anytime convergence.
}

\response{
\subsection{Maximum Safe Radius Computation}
Recent approaches to verification for neural networks can also be used to compute bounds on the maximum safety radius directly for, say, $L_\infty$, by gradually enlarging the region. These works can be classified into two categories. The first concentrates on estimation of the lower bound of $\maximumsaferadius$ using various techniques. For example, FastLin/FastLip \cite{Weng2018TowardsFC} and Crown \cite{Zhang2018EfficientNN} employ layer-by-layer analysis to obtain a tight lower bound by linearly bounding the ReLU (i.e., FastLin/FastLip) or non-linear activation functions (i.e., Crown).  Kolter\&Wong \cite{Kolter2018ProvableDA,NIPS2018_8060}, on the other hand, calculates the lower bound of $\maximumsaferadius$ by taking advantage of robust optimisation. The second category aims to adapt abstract interpretation techniques to prove safety. For example, DeepZ and DeepPoly \cite{Singh:2019:ADC:3302515.3290354} (also including $AI^2$) adapt abstract interpretation to perform layer-by-layer analysis to over-approximate the outputs for a set of inputs, so that some safety properties can be verified, but are unable to prove absence of safety.
A fundamental advantage of $\DeepGame$ compared to those works is that it can perform anytime estimation of $\maximumsaferadius$ by improving both lower and upper bounds monotonically, even with a loose Lipschitz constant. Moreover, $\DeepGame$ provides a theoretical guarantee that it can reach the exact value of $\maximumsaferadius$. 

Maximal radius computation for DNNs has been addressed directly in \cite{Carlini2017ProvablyMA,Tjeng2017EvaluatingRO}, where the entire DNN is encoded as a set of constraints, which are then solved by searching for valid solutions to the corresponding satisfiability or optimality problem. The approach of \cite{Carlini2017ProvablyMA} searches for a bound on the maximal safety radius
by utilising Reluplex and performing binary search, and \cite{Tjeng2017EvaluatingRO} instead considers an  MILP-based approach. In contrast, our approach utilises a Lipschitz constant to perform search over the input space. Further, our approach only needs to know the Lipschitz constant, whereas \cite{Carlini2017ProvablyMA,Tjeng2017EvaluatingRO} need access to the DNN architecture and the trained parameters.
}

\subsection{Black-box Algorithms}
The methods in \cite{papernot2017practical} evaluate a network  by generating a synthetic data set, training a surrogate model, and then applying white box detection techniques on the model. 
\cite{BlindSearchPaper} randomly searches the vector space around the input image for changes which will cause a misclassification. 
It shows that in some instances this method is 
efficient and able to indicate where salient areas of the image exist. \cite{RWSHKK2018} and this paper are black-box, except that grey-box feature extraction techniques are also considered in this paper to partition the input dimensions. L0-TRE~\cite{RWSHKK2018} quantifies the global robustness of a DNN, where global robustness is the expectation of the maximum safe radius over a testing dataset, through iteratively generating lower and upper bounds on the network's robustness.

\subsection{Feature Extraction Techniques}

Feature extraction is an active area of research in machine learning, where the training data  are usually sampled from real world problems and high dimensional. Feature extraction techniques reduce the dimensionality of the training data by using a set of features to represent an input sample. 
In this paper, feature extraction is used not for reducing dimensionality, but rather to partition the input dimensions into a small set of features. Feature extraction methods can be classified into those that are specific to the problem, such as the SIFT~\cite{SIFT}, SURF~\cite{SURF} and superpixels~\cite{wei2015superpixels}, which are specific to the object detection, and general methods, such as the techniques for computing for every input dimension its significance to the output \cite{LL2017}. The significance values can be visualised as a saliency map, as done in e.g., \cite{JSMA,RWSHKK2018}, but can also be utilised as in this paper to  partition the input dimensions.

\section{Conclusion}\label{sec:concl}

In this work, we present a two-player turn-based game framework for the verification of deep neural networks with provable guarantees. We tackle two problems, \emph{maximum safe radius} and \emph{feature robustness}, which essentially correspond to the absolute (pixel-level) and relative (feature-level) safety of a network against adversarial manipulations. Our framework can deploy various feature extraction or image segmentation approaches, including the saliency-guided $\mathsf{grey}$-box mechanism, and the feature-guided $\mathsf{black}$-box procedure. We develop a software tool $\DeepGame$, and demonstrate its applicability on state-of-the-art networks and dataset benchmarks. Our experiments exhibit converging upper and lower bounds, 
and are competitive compared to existing approaches to search for adversarial examples. 
Moreover, our framework can be utilised to evaluate robustness of networks in safety-critical applications such as traffic sign recognition in self-driving cars. 

\paragraph{Acknowledgements}
Kwiatkowska and Ruan are supported by the EPSRC Mobile Autonomy Programme Grant (EP/M019918/1). Huang gratefully acknowledges NVIDIA Corporation for its support with the donation of GPU, and is partially supported by NSFC (NO. 61772232). Wu is supported by the CSC-PAG Oxford Scholarship.

\bibliography{Bibliography}

\begin{thebibliography}{10}
\expandafter\ifx\csname url\endcsname\relax
  \def\url#1{\texttt{#1}}\fi
\expandafter\ifx\csname urlprefix\endcsname\relax\def\urlprefix{URL }\fi
\expandafter\ifx\csname href\endcsname\relax
  \def\href#1#2{#2} \def\path#1{#1}\fi

\bibitem{malware}
G.~Dahl, J.~W. Stokes, L.~Deng, D.~Yu, Large-scale malware classification using
  random projections and neural networks, in: Proceedings IEEE Conference on
  Acoustics, Speech, and Signal Processing, IEEE SPS, 2013.

\bibitem{ryan:nips10}
J.~Ryan, M.-J. Lin, R.~Miikkulainen, Intrusion detection with neural networks,
  in: M.~I. Jordan, M.~J. Kearns, S.~A. Solla (Eds.), Advances in Neural
  Information Processing Systems 10, Cambridge, MA: MIT Press, 1998, pp.
  943--949.

\bibitem{NVIDIA}
M.~Bojarski, D.~D. Testa, D.~Dworakowski, B.~Firner, B.~Flepp, P.~Goyal, L.~D.
  Jackel, M.~Monfort, U.~Muller, J.~Zhang, X.~Zhang, J.~Zhao, K.~Zieba, End to
  end learning for self-driving cars, CoRR.

\bibitem{road-segmentation}
S.~Bittel, V.~Kaiser, M.~Teichmann, M.~Thoma, Pixel-wise segmentation of street
  with neural networks, CoRR.

\bibitem{traffic-classification-lecun}
P.~Sermanet, Y.~LeCun, Traffic sign recognition with multi-scale convolutional
  networks, in: Neural Networks (IJCNN), The 2011 International Joint
  Conference on, IEEE, 2011, pp. 2809--2813.

\bibitem{LBH2015}
Y.~LeCun, Y.~Bengio, G.~Hinton, Deep learning, Nature 521 (2015) 436--444.

\bibitem{Biggio2013}
B.~Biggio, I.~Corona, D.~Maiorca, B.~Nelson, N.~Srndic, P.~Laskov, G.~Giacinto,
  F.~Roli, Evasion attacks against machine learning at test time, in: ECML/PKDD
  2013, 2013, pp. 387--402.

\bibitem{propertiesOfNeuralNetworks}
C.~Szegedy, W.~Zaremba, I.~Sutskever, J.~Bruna, D.~Erhan, I.~Goodfellow,
  R.~Fergus, Intriguing properties of neural networks, in: International
  Conference on Learning Representations, 2014.

\bibitem{FGSM}
I.~J. Goodfellow, J.~Shlens, C.~Szegedy, Explaining and harnessing adversarial
  examples, CoRR.

\bibitem{JSMA}
N.~Papernot, P.~McDaniel, S.~Jha, M.~Fredrikson, Z.~B. Celik, A.~Swami, The
  limitations of deep learning in adversarial settings, in: Security and
  Privacy (EuroS\&P), 2016 IEEE European Symposium on, IEEE, 2016, pp.
  372--387.

\bibitem{CW-Attacks}
N.~Carlini, D.~Wagner, Towards evaluating the robustness of neural networks,
  in: Security and Privacy (SP), 2017 IEEE Symposium on, IEEE, 2017, pp.
  39--57.

\bibitem{KBDJK2017}
G.~Katz, C.~Barrett, D.~L. Dill, K.~Julian, M.~J. Kochenderfer, Reluplex: An
  efficient smt solver for verifying deep neural networks, in: R.~Majumdar,
  V.~Kun{\v{c}}ak (Eds.), Computer Aided Verification, Springer International
  Publishing, Cham, 2017, pp. 97--117.

\bibitem{DLV}
X.~Huang, M.~Kwiatkowska, S.~Wang, M.~Wu, Safety verification of deep neural
  networks, in: R.~Majumdar, V.~Kun{\v{c}}ak (Eds.), Computer Aided
  Verification, Springer International Publishing, Cham, 2017, pp. 3--29.

\bibitem{old-lip-1}
R.~R. {Zakrzewski}, Verification of a trained neural network accuracy, in:
  IJCNN'01. International Joint Conference on Neural Networks. Proceedings
  (Cat. No.01CH37222), Vol.~3, 2001, pp. 1657--1662 vol.3.

\bibitem{old-lip-2}
J.~{Hull}, D.~{Ward}, R.~R. {Zakrzewski}, Verification and validation of neural
  networks for safety-critical applications, in: Proceedings of the 2002
  American Control Conference (IEEE Cat. No.CH37301), Vol.~6, 2002, pp.
  4789--4794 vol.6.

\bibitem{RHK2018}
W.~Ruan, X.~Huang, M.~Kwiatkowska, Reachability analysis of deep neural
  networks with provable guarantees, in: Proceedings, International Joint
  Conference on Artificial Intelligence ({IJCAI}), 2018.

\bibitem{RWSHKK2018}
W.~Ruan, M.~Wu, Y.~Sun, X.~Huang, D.~Kroening, M.~Kwiatkowska, Global
  robustness evaluation of deep neural networks with provable guarantees for
  the {L0} norm, arXiv preprint arXiv:1804.05805.

\bibitem{SIFT}
D.~G. Lowe, Distinctive image features from scale-invariant keypoints,
  International journal of computer vision 60~(2) (2004) 91--110.

\bibitem{MNIST}
Y.~LeCun, L.~Bottou, Y.~Bengio, P.~Haffner, Gradient-based learning applied to
  document recognition, Proceedings of the IEEE 86~(11) (1998) 2278--2324.

\bibitem{CIFAR10}
A.~Krizhevsky, Learning multiple layers of features from tiny images, Citeseer,
  2009.

\bibitem{GTSRB}
J.~Stallkamp, M.~Schlipsing, J.~Salmen, C.~Igel, Man vs. computer: Benchmarking
  machine learning algorithms for traffic sign recognition, Neural networks 32
  (2012) 323--332.

\bibitem{NexarData}
Nexar, Challenge: Using deep learning for traffic light recognition.
  {https://www.getnexar.com/challenge-1}.

\bibitem{WHK2017}
M.~Wicker, X.~Huang, M.~Kwiatkowska, Feature-guided black-box safety testing of
  deep neural networks, in: International Conference on Tools and Algorithms
  for the Construction and Analysis of Systems, Springer, 2018, pp. 408--426.

\bibitem{WSB2003}
Z.~Wang, E.~P. Simoncelli, A.~C. Bovik, Multiscale structural similarity for
  image quality assessment, in: Signals, Systems and Computers, Conference
  Record of the Thirty-Seventh Asilomar Conference on, 2003.

\bibitem{SZSBEGF2014}
C.~Szegedy, W.~Zaremba, I.~Sutskever, J.~Bruna, D.~Erhan, I.~Goodfellow,
  R.~Fergus, Intriguing properties of neural networks, 2014.

\bibitem{Carlsson2008}
G.~Carlsson, T.~Ishkhanov, V.~de~Silva, A.~Zomorodian, On the local behavior of
  spaces of natural images, International Journal of Computer Vision 76~(1)
  (2008) 1--12.

\bibitem{LIME}
M.~T. Ribeiro, S.~Singh, C.~Guestrin, "{W}hy should {I} trust you?": Explaining
  the predictions of any classifier, in: Proceedings of the 22nd ACM SIGKDD
  International Conference on Knowledge Discovery and Data Mining, KDD~'16,
  2016, pp. 1135--1144.

\bibitem{LL2017}
S.~Lundberg, S.-I. Lee, A unified approach to interpreting model predictions,
  in: NIPS2017, 2017.

\bibitem{SWRHKK2018}
Y.~Sun, M.~Wu, W.~Ruan, X.~Huang, M.~Kwiatkowska, D.~Kroening, Concolic testing
  for deep neural networks, arXiv preprint arXiv:1805.00089v1.

\bibitem{CWU2008}
G.~Chaslot, M.~Winands, J.~Uiterwijk, H.~van~den Herik, B.~Bouzy, Progressive
  strategies for monte-carlo tree search, New Mathematics and Natural
  Computation 4~(3) (2008) 343--359.

\bibitem{KS2006}
L.~Kocsis, C.~Szepesv{\'a}ri, Bandit based monte-carlo planning, in: European
  Conference on Machine Learning, Springer, 2006, pp. 282--293.

\bibitem{wei2015superpixels}
S.-C. Wei, T.-J. Yen, Superpixels generating from the pixel-based k-means
  clustering., Journal of Multimedia Processing and Technologies 6~(3) (2015)
  77--86.

\bibitem{Weng2018TowardsFC}
T.-W. Weng, H.~Zhang, H.~Chen, Z.~Song, C.-J. Hsieh, D.~S. Boning, I.~S.
  Dhillon, L.~Daniel, Towards fast computation of certified robustness for relu
  networks, in: ICML, 2018.

\bibitem{Zhang2018EfficientNN}
H.~Zhang, T.-W. Weng, P.-Y. Chen, C.-J. Hsieh, L.~Daniel, Efficient neural
  network robustness certification with general activation functions, in:
  NeurIPS, 2018.

\bibitem{NexarEntry}
A.~Burg, {Deep Learning Traffic Lights} model for {Nexar Competition}.
  https://github.com/burgalon/deep-learning-traffic-lights.

\bibitem{FGSM-universal}
S.-M. Moosavi-Dezfooli, A.~Fawzi, O.~Fawzi, P.~Frossard, Universal adversarial
  perturbations, in: Computer Vision and Pattern Recognition (CVPR), 2017 IEEE
  Conference on, IEEE, 2017, pp. 86--94.

\bibitem{DBLP:journals/corr/abs-1708-06939}
M.~Melis, A.~Demontis, B.~Biggio, G.~Brown, G.~Fumera, F.~Roli, Is deep
  learning safe for robot vision? adversarial examples against the icub
  humanoid, CoRR abs/1708.06939.

\bibitem{PT2010}
L.~Pulina, A.~Tacchella, An abstraction-refinement approach to verification of
  artificial neural networks, in: International Conference on Computer Aided
  Verification, Springer, 2010, pp. 243--257.

\bibitem{GKPB2017}
D.~Gopinath, G.~Katz, C.~S. Pasareanu, C.~Barrett, Deepsafe: A data-driven
  approach for checking adversarial robustness in neural networks, CoRR.

\bibitem{pmlr-v80-mirman18b}
M.~Mirman, T.~Gehr, M.~Vechev, Differentiable abstract interpretation for
  provably robust neural networks, in: J.~Dy, A.~Krause (Eds.), Proceedings of
  the 35th International Conference on Machine Learning, Vol.~80 of Proceedings
  of Machine Learning Research, PMLR, Stockholmsmässan, Stockholm Sweden,
  2018, pp. 3578--3586.

\bibitem{Kolter2018ProvableDA}
J.~Z. Kolter, E.~Wong, Provable defenses against adversarial examples via the
  convex outer adversarial polytope, in: ICML, 2018.

\bibitem{NIPS2018_8060}
E.~Wong, F.~Schmidt, J.~H. Metzen, J.~Z. Kolter, Scaling provable adversarial
  defenses, in: S.~Bengio, H.~Wallach, H.~Larochelle, K.~Grauman,
  N.~Cesa-Bianchi, R.~Garnett (Eds.), Advances in Neural Information Processing
  Systems 31, Curran Associates, Inc., 2018, pp. 8400--8409.

\bibitem{Singh:2019:ADC:3302515.3290354}
G.~Singh, T.~Gehr, M.~P\"{u}schel, M.~Vechev, An abstract domain for certifying
  neural networks, Proc. ACM Program. Lang. 3~(POPL) (2019) 41:1--41:30.

\bibitem{Carlini2017ProvablyMA}
N.~Carlini, G.~Katz, C.~W. Barrett, D.~L. Dill, Provably minimally-distorted
  adversarial examples, 2017.

\bibitem{Tjeng2017EvaluatingRO}
V.~Tjeng, K.~Xiao, R.~Tedrake, Evaluating robustness of neural networks with
  mixed integer programming, 2017.

\bibitem{papernot2017practical}
N.~Papernot, P.~McDaniel, I.~Goodfellow, S.~Jha, Z.~B. Celik, A.~Swami,
  Practical black-box attacks against machine learning, in: Proceedings of the
  2017 ACM on Asia Conference on Computer and Communications Security, ACM,
  2017, pp. 506--519.

\bibitem{BlindSearchPaper}
N.~Narodytska, S.~Kasiviswanathan, Simple black-box adversarial attacks on deep
  neural networks, in: 2017 IEEE Conference on Computer Vision and Pattern
  Recognition Workshops (CVPRW), 2017, pp. 1310--1318.

\bibitem{SURF}
H.~Bay, A.~Ess, T.~Tuytelaars, L.~Van~Gool, Speeded-up robust features (surf),
  Comput. Vis. Image Underst.~(3)  346--359.

\end{thebibliography}

\newpage
\appendix

\section{Experimental Setting for Comparison between $\DeepGame$ with Existing Works}
\label{app:comparison}

\subsection{Model Architectures}
\label{app:architectures}

\begin{table}[h!]
\caption{Architectures of the MNIST, CIFAR-10, and GTSRB models.}
\label{tbl:MNIST+CIFAR10}
\centering
\begin{tabular}{ l | l | l }
\toprule
Layer Type & MNIST & CIFAR-10/GTSRB \\
\hline
Convolution + ReLU & 3 $\times$ 3 $\times$ 32 & 3 $\times$ 3 $\times$ 64 \\ 
Convolution + ReLU & 3 $\times$ 3 $\times$ 32 & 3 $\times$ 3 $\times$ 64 \\ 
Max Pooling & 2 $\times$ 2 & 2 $\times$ 2 \\ 
Convolution + ReLU & 3 $\times$ 3 $\times$ 64 & 3 $\times$ 3 $\times$ 128 \\ 
Convolution + ReLU & 3 $\times$ 3 $\times$ 64 & 3 $\times$ 3 $\times$ 128 \\ 
Max Pooling & 2 $\times$ 2 & 2 $\times$ 2 \\ 
Flatten &  &  \\ 
Fully Connected + ReLU & 200 & 256 \\ 
Dropout & 0.5 & 0.5 \\ 
Fully Connected + ReLU & 200 & 256 \\ 
Fully Connected + Softmax & 10 & 10 \\ 
\bottomrule    
\end{tabular}
\end{table}

\response{
\begin{itemize}
    \setlength\itemsep{0em}
	\item Training Accuracy: 
	\begin{itemize}
	    \setlength\itemsep{0em}
		\item MNIST (99.99\% on 60,000 images)
		\item CIFAR-10 (99.83\% on 50,000 images)
	\end{itemize}
	\item Testing Accuracy:
	\begin{itemize}
	    \setlength\itemsep{0em}
		\item MNIST (99.36\% on 10,000 images)
		\item CIFAR-10 (78.30\% on 10,000 images)
	\end{itemize}
\end{itemize}
}

\subsection{Datasets}

We perform the comparison on two datasets: MNIST and CIFAR-10. They are standard benchmark datasets for adversarial attack of DNNs, and are widely adopted by all these baseline methods.

\begin{itemize}
    \item MNIST dataset\footnote{\url{http://yann.lecun.com/exdb/mnist/}}: an image dataset of handwritten digits, which contains a training set of 60,000 examples and a test set of 10,000 examples. The digits have been size-normalised and centred in a fixed-size image.
    \item CIFAR-10 dataset\footnote{\url{https://www.cs.toronto.edu/~kriz/cifar.html}}: an image dataset of 10 mutually exclusive classes, i.e., `airplane', `automobile', `bird', `cat', `deer', `dog', `frog', `horse', `ship', `truck'. It consists of 60,000 $32\times32$ colour images, with 50,000 for training, and 10,000 for testing. 
\end{itemize}

\subsection{Baseline Methods}
We choose a few well-established baseline methods that can perform state-of-the-art $L_0$ adversarial attacks. Their codes are all available on GitHub.

\begin{itemize}
    \item CW\footnote{\url{https://github.com/carlini/nn_robust_attacks}}: a state-of-the-art adversarial attacking method, which models the attacking problem as an unconstrained optimization problem that is solvable by Adam optimizer in TensorFlow.
    \item L0-TRE\footnote{\url{https://github.com/TrustAI/L0-TRE}}: a tensor-based robustness evaluation tool for the L0-norm, and its competitive L0 attack function is compared in this work.
    \item DLV\footnote{\url{https://github.com/TrustAI/DLV}}: an untargeted DNN verification method based on exhaustive search and MCTS.
    \item SafeCV\footnote{\url{https://github.com/matthewwicker/SafeCV}}: a feature-guided safety verification work based on SIFT features, game theory, and MCTS.
    \item JSMA\footnote{\url{https://github.com/bethgelab/foolbox}}: a targeted attack based on $L_0$-norm, so we perform this attack in a sense that the adversarial examples are misclassified into all classes except the correct one.
\end{itemize}

\subsection{Parameter Setting}

MNIST and CIFAR-10 use the same settings, unless separately specified.
\begin{itemize}
    \setlength\itemsep{0em}
    
    \item $\DeepGame$
    \begin{itemize}
    \setlength\itemsep{0em}
        \item gameType = `cooperative'
        \item bound = `ub'
        \item algorithm = `A*'
        \item eta = (`L0', 30)
        \item tau = 1
    \end{itemize}
    
    \item CW: 
    \begin{itemize}
    \setlength\itemsep{0em}
        \item targeted = False
        \item learning\_rate = 0.1
        \item max\_iteration = 100
    \end{itemize}
    
    \item L0-TRE: 
    \begin{itemize}
    \setlength\itemsep{0em}
        \item EPSILON = 0.5
        \item L0\_UPPER\_BOUND = 100
    \end{itemize}
    
    \item DLV: 
    \begin{itemize}
    \setlength\itemsep{0em}
        \item mcts\_mode = ``sift\_twoPlayer"
        \item startLayer, maxLayer = -1
        \item numOfFeatures = 150
        \item featureDims = 1
        \item MCTS\_level\_maximal\_time = 30
        \item MCTS\_all\_maximal\_time = 120
        \item MCTS\_multi\_samples = 5 (MNIST), 3 (CIFAR-10)
    \end{itemize}

    \item SafeCV:
    \begin{itemize}
    \setlength\itemsep{0em}
        \item MANIP = max\_manip (MNIST), white\_manipulation (CIFAR-10)
        \item VISIT\_CONSTANT = 1
        \item backtracking\_constant = 1
        \item simulation\_cutoff = 75 (MNIST), 100 (CIFAR10)
        \item small\_image = True
    \end{itemize}
    
    \item JSMA: 
    \begin{itemize}
    \setlength\itemsep{0em}
        \item bounds = (0, 1)
        \item predicts = `logits'
    \end{itemize}
\end{itemize}

\subsection{Platforms}

\begin{itemize}
 \item Hardware Platform:
    \begin{itemize}
        \setlength\itemsep{0em}
        \item NVIDIA GeForce GTX TITAN Black
        \item Intel(R) Core(TM) i5-4690S CPU @ 3.20GHz $ \times $ 4
    \end{itemize}
 
 \item Software Platform: 
    \begin{itemize}
        \setlength\itemsep{0em}
        \item Ubuntu 14.04.3 LTS
        \item Fedora 26 (64-bit)
        \item Anaconda, PyCharm
    \end{itemize}
\end{itemize}

\subsection{Adversarial Images}

\begin{figure}[t]
	\centering
	\includegraphics[width=0.8\linewidth]{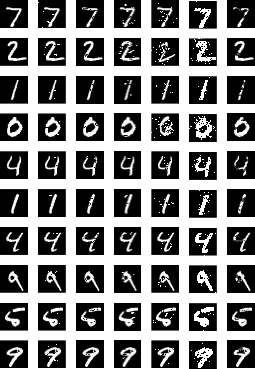}
	\caption{Comparison of the generated adversarial MNIST images (see Section~\ref{subsec:L0Comparison}). From left to right: original image, $\DeepGame$ (this paper), CW, L0-TRE, DLV, SafeCV, and JSMA.}
	\label{fig:adversary+mnist}
\end{figure}

\begin{figure}[t]
	\centering
	\includegraphics[width=0.9\linewidth]{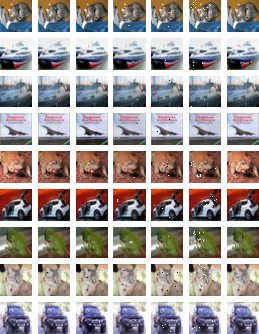}
	\caption{Comparison of the generated adversarial CIFAR10 images (see Section~\ref{subsec:L0Comparison}. From left to right: original image, $\DeepGame$ (this paper), CW, L0-TRE, DLV, SafeCV, and JSMA.}
	\label{fig:adversary+cifar10}
\end{figure}

\begin{figure}[t]
	\centering
	\includegraphics[width=1\linewidth]{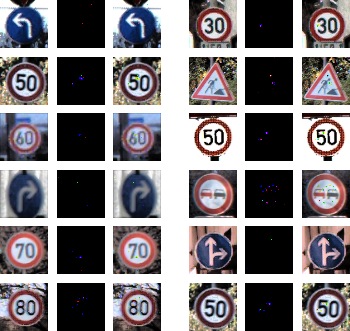}
	\caption{Adversarial GTSRB images generated by our tool $\mathsf{DeepGame}$.}
	\label{fig:adversary+gtsrb}
\end{figure}

\end{document}